\documentclass[12pt]{article}
\usepackage{authblk}
\usepackage{graphicx}
\usepackage{caption}
\usepackage{subcaption}
\usepackage{booktabs}
\usepackage{algorithm}
\usepackage[algo2e]{algorithm2e}
\RestyleAlgo{ruled}
\usepackage{natbib}

\usepackage{amsmath}
\usepackage{amssymb}
\usepackage{amsthm}
\usepackage{dsfont}

\newtheorem{theorem}{Theorem}
\newtheorem{lemma}{Lemma}
\newtheorem{corollary}{Corollary}
\newtheorem{proposition}{Proposition}

\usepackage{color}
\usepackage{shortcuts}
\begin{document}

%

%

\title{Fast kernel half-space depth for data with non-convex supports}



\author[1]{ Arturo Castellanos}
\author[1]{ Pavlo Mozharovskyi}
\author[1]{ Florence d'Alché-Buc}
\author[1]{ Hicham Janati}

\affil[1]{ LTCI, T{\'e}l{\'e}com Paris, Institut Polytechnique de Paris }

\maketitle

\begin{abstract}

Data depth is a statistical function that generalizes order and quantiles to the multivariate setting and beyond, with applications spanning over descriptive and visual statistics, anomaly detection, testing, etc. The celebrated halfspace depth exploits data geometry via an optimization program to deliver properties of invariances, robustness, and non-parametricity. Nevertheless, it implicitly assumes convex data supports and requires exponential computational cost. To tackle distribution's multimodality, we extend the halfspace depth in a Reproducing Kernel Hilbert Space (RKHS). We show that the obtained depth is intuitive and establish its consistency with provable concentration bounds that allow for homogeneity testing. The proposed depth can be computed using manifold gradient making faster than half-space depth by several orders of magnitude. The performance of our depth is demonstrated through numerical simulations as well as applications such as anomaly detection on real data and homogeneity testing.
\end{abstract}

\section{INTRODUCTION}

    Quantiles are a staple for statistics, but they are well-defined only in one dimension. Different attempts to generalize them to multiple dimensions have emerged in the last few decades, and among them, data depths~\citep{mosler2013depth}. Data depths are functions that measure the centrality of data points for given a probability distribution thereby providing a way to rank them. The field of data depths is rich of various propositions, such as halfspace depth~\citep{tukey1975mathematics}, Mahalanobis depth~\citep{Mahalanobis36} and projection depth~\citep{ZuoS00a} to name a few.
    Vast applications range from anomaly detection~\citep{RousseeuwH18,staerman2023functional}, classification via DD-plot~\citep{lange2014fast} to generalization of rank tests to the multivariate setting~\citep{liu1993quality,zuo2006limiting}.
    

Unlike the univariate case, multidimensional spaces lack an obvious axis of reference by which data points can be ranked. A straightforward attempt to generalize ranking is to compute the data projections on \emph{all}
possible axes and combining them with an aggregation function.  Taking their minimum for instance leads to the classical half-space depth $\halfspacedepth$. While $\halfspacedepth$ is robust to data contamination, it suffers from two main limitations: (1) it assumes that the distributions have convex supports; (2) it requires solving a non-convex and non-smooth optimization problem for which zero-order algorithms are the only available option.

    \paragraph{Related work}
        To tackle non-convexity, the Monge-Kantarovich depth~\citep{chernozhukov2017monge,del2018center}, proposes to solve an optimal transport problem between a well-known distribution of reference whose (halfspace) data depth regions are theoretically known (for instance, the uniform distribution on the unit ball) and the distribution of interest.
        However, it is not adapted to multimodal distributions since to map multiple modes, the transport map cannot be continuous.

    To tackle the complexity of such challenging distributions, kernel methods~\citep{scholkopf2002learning,hofmann2008kernel} have been one of the tools of predilection in machine learning. This led to the proposal of several kernel-based depth functions such as the  kernelized spatial depth~\citep{chen2008outlier} and the localized spatial depth (LSPD)~\citep{dutta2016multi} which are kernelized versions of the spatial depth~\citep{vardi2000multivariate}. Another possible definition of a kernel-based data depth can be obtained from the One-Class Support Vector Machine (OC-SVM)~\citep{scholkopf2001estimating, vert2006consistency} using its boundary decision function as a depth proxy. This leads to a learned depth function for the entire space through solving a single optimization problem but at the expense of robustness. Moreover, while all kernel-based methods can tackle highly non-linear data, tuning the kernel scale hyper-parameter can be cumbersome in practice, specially in unsupervised settings where no cross-validation can be carried out.
    
    Ideally, the ultimate data depth function should: (1) be robust i.e the depth of a data observation should not be influenced by adding or removing any outliers to the data; (2)
    adapted for non-convex supports; (3) scalable, i.e defined through an easy to solve differentiable optimization problem.
    
    While each of the aforementioned depth functions have their advantages, they fail to unite all the three properties mentioned above. Indeed, most robust depths take their robustness from the fact that for each point they compute a non-smooth optimization program with respect to the data (geometry). This robustness comes at the cost of either a high computational complexity of a difficult optimization problem, or a simple optimization problem with high restrictions on the nature of the distributions.

    \paragraph{Contributions}

    In this work, we propose a more natural extension of the half-space depth through kernel methods. By considering scalar products in a Reproducing Kernel Hilbert space associated with a radial-basis kernel, we show that the proposed depth amounts to taking local spherical projections of the data, hence the name \emph{Sphere depth}. The obtained optimization problem is non-differentiable and constrained on the unit sphere of $\bbR^d$. We propose a differentiable relaxation using the sigmoid function which allows, to our knowledge, for the first proposal of a consistent depth function computed through manifold gradient descent. 

Our contributions are as follows:
\begin{enumerate}
\item Inspired by kernel methods, we propose a novel depth function $\SphereD_s$ adapted for data with non-convex supports.
\item We propose a fast Riemannian gradient descent algorithm to compute $\SphereD_s$.
\item We prove that $\SphereD_s$ is consistent and provide asymptotic concentration bounds that allow for statistical homogeneity testing.
\item We confirm the theoretical properties of $\SphereD_s$ through experiments with both simulated and real anomaly detection data.
\end{enumerate}


    In Section 2, we present the natural extension of the halfspace depth to a RKHS and the challenges that come with it, leading to our definition of Sphere depth. In Section 3, we prove our main theoretical results: namely consistency with concentration bounds. In Section 4, we present  our manifold gradient descent algorithm and illustrate its speed gain compared to the halfspace depth empirically. Finally, Section 5, the performance of the proposed depth is illustrated on both simulated and real data applications.

    \paragraph{Notation}
Let $\bX$ be a probability distribution in $\cP(\bbR^d)$. Along the paper, $X$ denotes a random variable distributed according to $\bX$ i.e $X \sim \bX$ and $\samples=\{x_1, \ldots, x_n\}$ corresponds to i.i.d. samples drawn from $\bX$. $\|.\|$ denotes the Euclidean norm 2. We will denote $\sphere(z,r,\|\cdot\|)$ (respectively $\ball(z,r,\|\cdot\|)$) the sphere (respectively ball) of center $z$ and radius $r$. For the Euclidian norm, we will just note $\sphere(z,r)$ and $\ball(z,r)$ for short.
The sigmoid function is denoted and defined as
       $sig_s: x \mapsto \frac{1}{1+e^{-x/s}}$.
In particular, we will denote $sig$ for $s=1$.
$\mathcal{N}(\mu,\Sigma)$ refers to the Normal distribution with mean $\mu$ and covariance $\Sigma$.

\section{HALFSPACE DEPTH IN THE RKHS}

\subsection{Kernelized halfspace depth}
\paragraph{Halfspace depth}
    Let $\bX$ be a probability distribution in $\cP(\reals^d)$, and $X \sim \bX$. When $d=1$, a point $z \in \bbR$ is \emph{deep} within $\bX$ if $\bX$ is ``dense" both on the left and on the right of $z$. Formally, Tukey's halfspace depth is given by $\min_{u =\{-1,+1\} } \mathbb{P}(u(X-z))$. Its generalization to $d>1$ is directly given by:
        \begin{align}
        \TukeyD(z|\bX) = \inf_{u \in \sphere(0,1)} \mathbb{P}(\langle u, X\rangle \geq \langle u, z\rangle). \label{generalTukey}
    \end{align}





    The Tukey depth $HD$ defined in \eqref{generalTukey} provides a simple measure of centrality. However, it suffers from a major limitation: it is not suited for distributions with non convex supports. Indeed, since it is based on Euclidean scalar products, i.e., linear projections, it cannot capture centrality for multimodal distributions for example (see Figure \ref{fig:intuition} (c)). 
    \paragraph{Kernelized halfspace depth}
    Here, we propose to extend the depth's definition of \eqref{generalTukey} by taking the inner product in a Reproducing Kernel Hilbert Space (RKHS) $\cH$ induced by a positive definite symmetric kernel $k: \cX \times \cX \to \reals$ with canonical feature map $\varphi:\cX \to \cH$ (see reminder in Appendix, Section~\ref{app:reminder}). Let $S_\mathcal{H}$ be a ball in the RKHS $\cH$. Such a generalization can be provided by:
        \begin{equation}
        \label{eq:infeasible_depth}
       \inf_{f \in S_\mathcal{H}} \bbP\left(\langle f,\varphi(X) \rangle_{\cH} \geq \langle f, \varphi(z)\rangle_{\cH}\right)
    \end{equation}

    However, the feasible space of this functional optimization problem is too large and leads to a degenerate depth function that collapses to zero with probability one, as pointed out in~\citep{generalFDA, dutta2011}. This is the case for instance when the RKHS is dense in the space of continuous functions i.e if the kernel is universal~\citep{steinwart2008support}, which holds for the Gaussian kernel for instance. To avoid overfitting, we can restrict the search space to 
    functions of the form $f = \varphi(c)$ 
    with $c \in \cC \subset \bbR^d$, $\cC$ compact. With a such a parametrization, notice that for any radial basis kernel taking $c=z$ provides a trivial solution. Indeed it holds:
    \begin{align*}
            \forall x\neq z, \langle \varphi(z),\varphi(z) \rangle_{\cH} > \langle \varphi(z),\varphi(x) \rangle_{\cH}
        \enspace.
        \end{align*}
    Thus, for $c$ close enough to $z$, the probability $\bbP\left(k(c,X)  \geq k(c, z)\right)$ would still collapse to zero. 

\paragraph{Kernelized depth is sphere depth}
    To circumvent the limitations  mentioned above, we set the constraint set $\cC$ to a sphere of radius $r$ centered at $z$ denoted by $\bbS(z, r)$ with a  hyperparameter $r > 0$. Formally, we introduce the kernelized half-space depth function:
    \begin{equation}
     \label{eq:kernel_depth}
        \SphereD(z|\bX) \eqdef \inf_{c \in \bbS(z,r)} \bbP\left(k(c, X)  \geq k(c, z)\right)
    \end{equation}
        The kernelized depth of \eqnref{eq:kernel_depth} 
        has two main intuitive advantages: (1) the data projections are non-linear and can be adapted for distributions with non-convex contours; (2) the additional parameter $r$ provides a flexible lever to control the depth's sensitivity depending on the data. These features becomes self-evident when $k$ is a radial-basis kernel. Taking the Gaussian kernel for instance leads to the following proposition:
%
%
    \begin{proposition}
    \label{prop:sphere}
     Let $k$ be the Gaussian kernel $k(x, y) \eqdef e^{-\gamma\|x - y\|^2}$ with $\gamma > 0.$ The kernel depth is independent of $\gamma$ and can be written:
        \begin{equation}\label{eq:ball-form}
            \SphereD(z|\bX) = \inf_{c\in \sphere(z,r)} \mathbb P_\bX(\bbB(c,r))
        \end{equation}
    \end{proposition}
        \begin{proof}
    For any $x, z \in \mathbb R^d$ and $c \in \bbS(z, r)$:
        \begin{equation*}
            k(c,x)\geq k(c,z) 
            \Leftrightarrow e^{-\gamma||x-c||^2} \geq e^{-\gamma r^2}
            \Leftrightarrow ||x-c|| \leq r.
        \end{equation*}
    Therefore $\bbP(k(c,\bX)\geq k(c,z))= \bbP_\bX(\bbB(c,r)) \qed$
    \end{proof}
The intuition provided by proposition \ref{prop:sphere} 
    is illustrated in Figure \ref{fig:intuition}. In practice, for the Gaussian kernel (or any radial basis kernel), computing $\SphereD(z|\samples)$ amounts to finding the ball of radius $r$ with the least amount of data observations, centered at a distance $r$ from $z$. For the rest of the paper, we will consider the radial-basis function kernel or Gaussian kernel defined by $ k(x,y) = e^{-\gamma||x-y||^2}.$
    
    \paragraph{The proposed depth}
    The sphere depth of proposition \ref{prop:sphere} can be written as:
    \begin{equation}
        \label{eq:sphere_depth_0}
        \SphereD(z | \bX) = \inf_{c\in \bbS(z, r)} \bbE \left[ \mathds 1_{\{r^2 - \|X - c\|^2 \geq 0 \}}\right]
     \end{equation}
    Problem \eqref{eq:sphere_depth_0} is a non-smooth and non-convex optimization problem that can only be solved  using gradient-free optimization algorithms (such as Nelder-Mead \citep{nelder1965simplex}) which do not scale well with $n$ or $d$. To tackle this limitation, we propose to smooth the indicator function using the sigmoid function and solve:
    \begin{equation}
        \label{eq:sigmoid_sphere_depth}
        \SphereD_s(z | \bX) \eqdef \inf_{c\in \bbS(z, r)} \bbE \left[ sig_s\left(r^2 - \|X - c\|^2\right)\right]
     \end{equation}

    This novel differentiable formulation has several advantages. First, it allows for a fast manifold gradient descent which will be detailed in Section \ref{s:algorithm}. Second, it provides a better depth ranking since its depth values are smoothed and not constrained to multiples of $\frac{1}{n}$. Third, the behavior of the original sphere depth can still be recovered with small values of $s$ since  $\lim_{s \to 0} \SphereD_s(z | \bX) = \SphereD(z | \bX)$ . In the rest of this paper, we denote $\SphereD_0$ and $\SphereD$ interchangeably to refer to the non-smoothed depth of \eqnref{eq:sphere_depth_0}.

\begin{figure}
 \includegraphics[width=\linewidth]{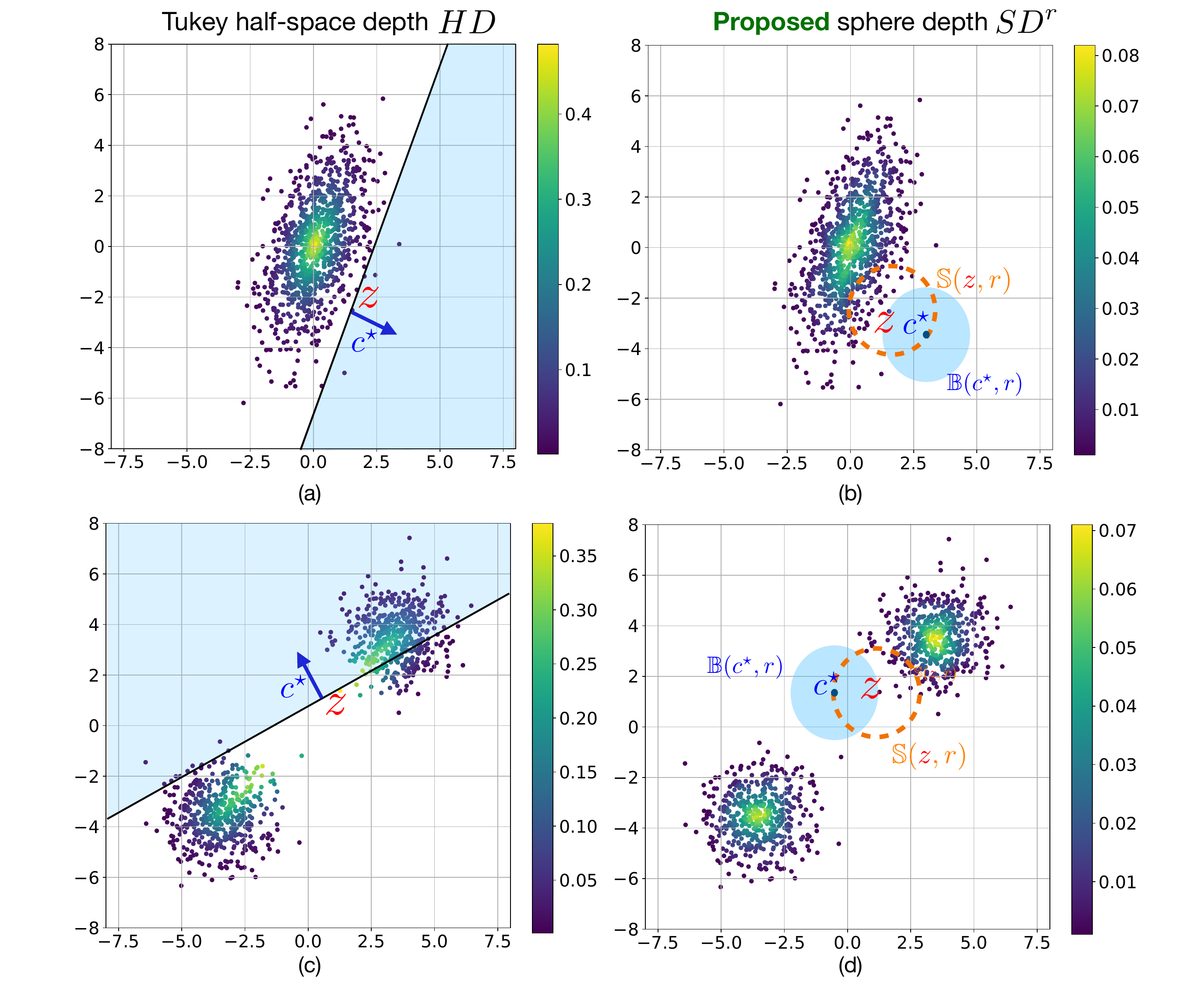}
 \caption{Example with toy data illustrating how the Tukey's half-space depth fails for data with non-convex supports, whereas the proposed Kernel approach accurately captures the depth of the data. All colorbars display the computed depth value: $HD$ \textbf{(left)} and $SD^r$ \textbf{(right)}. \label{fig:intuition}}
 \end{figure}

\subsection{Properties of the sphere depth}\label{subsec:depth-geom-prop}
In this section, we establish some first preliminary properties verified by the sphere depth.
    
\begin{proposition}\label{prop:depth-geom-properties}
        Let $z\in\bbR^d$ and $\bX$ a probability distribution in  $\cP(\bbR^d)$ and $X \sim \bX$. For any $s \geq 0$:
        \begin{itemize}
            \item (i)  $\SphereD_s(z|\bX)\in [0,1]$ and vanishes at infinity:\\ $\SphereD_s(z|\bX)\to 0$ when $||z||\rightarrow +\infty$,
            \item (ii) $\SphereD_s(z|\bX)$ is invariant by any global isometry $I$ of $\bbR^d$: $\SphereD_s(I(z)|I(\bX))=\SphereD_s(z|\bX)$.
        \end{itemize}
    \end{proposition}

    See Appendix, Section~\ref{app:proofs} for proof.

    We also show how to relate scaling to a change of parameters:

    \begin{proposition}\label{prop:scaling}
    For any $z\in\bbR^d$ and $\bX$ a probability distribution, $\lambda, s>0$:
    \begin{align*}
        SD^{\lambda r}_0 (\lambda z | \lambda X) = \SphereD_{0} (z|\bX)\\
        SD^{\lambda r}_s (\lambda z | \lambda X) = SD^{r}_{\frac{s}{\lambda^2}} (z|\bX)
    \end{align*}
    \end{proposition}
   
   It is interesting to note that the non-relaxed sphere depth $SD_0$ can easily be compared to the original halfspace depth. Indeed, each ball tangent to $z$ is contained in some halfspace. Formally:
    \begin{proposition}\label{prop:SD-less-than-HD}
        For any $z\in\bbR^d$ and any $r > 0$:
        \begin{equation*}
            \SphereD_0 (z|\bX) \leq \TukeyD (z|\bX)
        \end{equation*}
    \end{proposition}
    
    Moreover, equality holds when $r \to +\infty$ under some regular conditions, as the ball would cover almost all the halfspace except the hyperplane itself. However, if the data is entirely supported in a subspace of dimension strictly inferior to $d$, the extra dimension can be use to make the Sphere depth collapses where the halfspace depth would not . Dimension reduction when cleaning the data can help avoid this, and most of all sigmoid smoothing can prevent this kind of degeneracy.
    
\section{ASYMPTOTIC AND FINITE-SAMPLE RESULTS}\label{sec:properties}
\paragraph{Notation and definitions}
Let $(E, \|. \|)$ be a normed space and $K \subset E$.
A covering of radius $\varepsilon$ of $K \subset E$ is a set $C$ such that $K\subseteq \cup_{x\in C} \bbB(x,\varepsilon,||\cdot||)$. A covering which has the minimal number of elements among all possible covering is called a minimal covering. The cardinal of a minimal covering of $K$ is called the covering number and is noted $N(\varepsilon,K,||\cdot||)$.
    In a similar fashion, for a collection of functions $\mathcal{F}\subseteq\{\mathcal{X}\rightarrow \bbR\}$
    whose domain is a space $\cX$ and for a measure $\mu$ on $\cX$, a set ${[l_i,u_i]}^N_{i=1} \subseteq \{\mathcal{X}\rightarrow \bbR\}$ is called an $\varepsilon$-bracket of $\cF$ in $L_p(\mu)$, $p\geq 1$ if
    $
        \forall f \in \mathcal{F}, \exists i, l_i \leq f \leq u_i \text{ and } ||u_i-l_i||_{L_p(\mu)} \leq \varepsilon
    $.
    The bracketing number of $\cF$ is then the number of elements of a minimal $\varepsilon$-bracket which is of minimal cardinal among all possible $\varepsilon$-brackets of $\mathcal{F}$, and this number is noted $N_{[]}(\varepsilon,\cF,L_p(\mu))$
    Given a set $Z_{1:n}=(z_1,..,z_n)$ of $n$ sample points, we denote the Rademacher complexity of a function class $\mathcal{F}$ as:
    $
          \Radem_{Z_{1:n}}(\mathcal{F})=\frac{1}{n}\mathbb{E}_{\sigma_1,...,\sigma_n} \sup_{\mathcal{F}} \sum_{i=1}^n \sigma_i f(z_i)  
    $.
,
    where the $\sigma_i\sim Rademacher$ ($\mathbb{P}(\sigma_i=-1)=\mathbb{P}(\sigma_i=+1)=1/2$) independently. We note $\Radem_{P,n}(\mathcal{F})  =\mathbb{E}_{Z \sim P^n}[\Radem_{Z}(\mathcal{F})] $ the Rademacher complexity when we know the $Z_i$ are i.i.d. according to $P$.  

    We have now all the necessary tools to present our main theoretical results.
    
    \subsection{Consistency and concentration bounds}
    \label{ss:consistency}
    Our main theoretical contribution is summarized in the following theorem which establishes that $\SphereD_s$ is consistent with an exponentially decreasing concentration bound. 
    \begin{theorem}
    \label{thm:main}
    For any $r, s > 0$, the sphere depth $\SphereD_s$ verifies:
        \begin{equation}
        \label{eq:consistency}
             \underset{n\to+\infty} {\lim} |\SphereD_s(z|\samples)- \SphereD_s(z|\bX)| = 0
        \end{equation}
    and:
        \begin{equation}
                \label{eq:concentration}
            \bbP(|\SphereD_s(z|\samples)-\SphereD_s(z|\bX)|> R_{\bP_{\bX},n}(\cF_z)+t)\leq 2 e^{-\frac{nt^2}{2}}
        \end{equation}
    Moreover, if $\bX$ is absolutely continuous with a bounded support then for all $\varepsilon>0$:
        \begin{equation}
        \label{eq:consistency_support}
             \underset{n\to+\infty}{\lim} \bbP(\underset{z\in supp(\bX)}{\sup}|\SphereD_s(z|\samples)- \SphereD_s(z|\bX)| > \varepsilon) = 0
        \end{equation}
    \end{theorem}
    \label{sub:consistency}
    \subsection{Proof of Theorem \ref{thm:main}}
In this section we provide the sketch of proof of the results of our Theorem \ref{thm:main}. The full demonstration with technical details is provided in the Appendix, Section~\ref{app:proofs}.
\paragraph{Proof of consistency} To show the consistency of $\SphereD_s$, it will be useful to rewrite the data depth as:
    \begin{align*}
        \SphereD_s(z|\bX) &= \inf_{f\in \cF_z} \bbE[-f(X)] \\
                &= -\sup_{f\in \cF_z} \bbE [f(X)] 
    \end{align*}
    where $\mathcal{F}_z$ is the set of functions:
    \begin{equation*}
        \cF_z = \{ f_{z,c}: x \mapsto -sig_s\left(r^2 - \|x - c\|^2\right) | c\in \bbS(z,r) \}
    \end{equation*}
Therefore, it holds:
    \begin{align}
        |\SphereD_s(z|\samples)- &\SphereD_s(z|\bX)|  \nonumber \\ &= \lvert \sup_{f\in \cF_z} \bbE_\bX[f(X)]-\sup_{f\in \cF_z} \bbE_{\samples}[f(X)] \rvert  \nonumber \\ 
                        &\leq \sup_{f\in \cF_z} \lvert \bbE_\bX[f(X)]- \bbE_{\samples}[f(X)] \rvert \label{eq:depth2f}
    \end{align}
    Therefore, to prove the consistency of $\SphereD_s$ it is sufficient to establish the Glivenko-Cantelli property for the class of functions $\cF_z$. To do so, we will use the following well-known characterisation via its bracketing number:

    \begin{theorem}~\cite[Th 2.4.1]{van1996weak}\label{thm:GC}
        Let $\cF$ be a class of measurable functions such that, for every $\varepsilon > 0$, its bracketing number is finite : $N_{[]}(\varepsilon,\mathcal{F},L_1(P)) < \infty$ . Then $\mathcal{F}$ is Glivenko-Cantelli. 
    \end{theorem}
    First remark that the functions of $\cF_z$ are well-behaved thanks to the sigmoid function:
    \begin{proposition}\label{prop:lip}
        The functions $f_{z,c}$ of $\mathcal{F}_z$ are Lipschitz (with respect to $x$).
        Symmetrically, they are also Lipschitz with respect to $c$.
    \end{proposition}
    \begin{proofsketch}
         The gradient of any element of $\cF_z$ is easily upper-bounded. See the Appendix, Section~\ref{app:proofs} for the full proof.
    \end{proofsketch}%
    The Lipschitzness of the sigmoid 
    allows us to reduce the bracketing number of $\cF_z$ to the covering number of the sphere $\sphere(z,r)$. The covering number of a Euclidian ball is also a classic result (see for instance the book
    by~\citet{vershynin2018high}):
    \begin{lemma}~(\citet{vershynin2018high})
    \label{lem:ballcover}
        The covering number of the Euclidian unit ball $B$ in dimension d is of order:
        $N(\varepsilon,B,||\cdot||_2) = \Theta\left(\frac{1}{\varepsilon^d}\right)$.
    \end{lemma}
    By dilation of the radius $r$, we obtain the following result:

    \begin{proposition}\label{prop:bracketnum}
    The bracketing number of $\cF_z$ is bounded: $N_{[]}(\varepsilon,\mathcal{F}_z,L_1(P)) <\infty$ for any $\varepsilon>0$. In particular, for $\varepsilon$ sufficiently small, 
        $N_{[]}(\varepsilon,\mathcal{F}_z,L_1(P)) \leq O\left((\frac{r}{\varepsilon})^d\right)$, otherwise 
        $N_{[]}(\varepsilon,\mathcal{F}_z,L_1(P)) = O(1)$.
    \end{proposition}
    Therefore, the assumptions of Theorem~\ref{thm:GC} hold and we can conclude that $\cF_z$ is a class of Glivenko-Cantelli functions, thus the consistency follows.
    \paragraph{Proof of concentration bounds}
    First, we deduce bounds on the Rademacher complexity of $\mathcal{F}_z$ using (see again the same book of reference~\citep{vershynin2018high}):
    \begin{theorem}[Chaining, Dudley]\label{thm:dudley}
    For any distribution $P$, the Rademacher complexity of a function class $\cF$ is bounded, up to a constant factor $C$, by the following integral:
    \begin{equation*}
        R_{P,n}(\mathcal{F}) \leq C \int_0^{\infty} \sqrt{\frac{\log N(\varepsilon,\mathcal{F},L_2(P))}{n}} d\varepsilon
    \end{equation*}
    \end{theorem}
    
    Combining Proposition~\ref{prop:bracketnum} and Dudley's theorem just above, we obtain the following result:
    \begin{proposition}\label{prop:rademacher-fz}
    The Rademacher complexity of $\cF_z$ with respect to some distribution $P$ is of order:
    \begin{equation*}
        R_{P,n}(\mathcal{F}_z) = O\left(\sqrt{\frac{d}{n}}\right)
    \end{equation*}
    \end{proposition}

    Using tools from, e.g. the book of ~\citep{vershynin2018high}, it is well-known that the Rademacher complexity, under some bounding assumption, admits subgaussianity concentration:
    
    \begin{proposition}
    \label{prop:bounded-rademacher}
        For $\mathcal{F}$ a space of functions that are bounded by $b$ (in $||\cdot||_{\infty}$)
        \begin{equation*}
            \bbP(\sup_{f\in \mathcal{F}} \lvert \mathbb{E}_{\bX}[f(X)]- \mathbb{E}_{X_n}[f(X)] \rvert>\Radem_{\bX,n}(\mathcal{F})+t)\leq 2 e^{-\frac{nt^2}{2b^2}}
        \end{equation*}
    \end{proposition}
    Combining Proposition~\ref{prop:bounded-rademacher} with \eqnref{eq:depth2f} leads to the concentration bound of \eqnref{eq:concentration}.
%

\paragraph{Proofs for continuous distributions with bounded supports}
    The concentration bound of \eqnref{eq:concentration} can be extended from a single point to points on a bounded set under some assumptions. To achieve this, we first need some intermediary result:

    \begin{proposition}\label{prop:lipschitz-depth}
        The Sphere depth $z \to \SphereD_s(z|\bX)$ is Lipschitz with respect to $z$.
    \end{proposition}  
    \begin{proofsketch}
        For another point $z'$ obtained from $z$ by translation by $h=z'-z$, any point $c'\in \bbS(z',r)$ corresponds to a translation by $h$ of a point $c \in \bbS(z,r)$. Then we use the Lipschitzness of the function $f_z,c$ with respect to $c$ by Proposition~\ref{prop:lip}, and pass to the infinum to conclude.
    \end{proofsketch}
  
    The Lipschitzness and the concentration with subgaussian rate of \eqnref{eq:concentration} allow to extend the consistency to a bounded set of $\bbR^d$ via a covering of it. 

    \begin{corollary}\label{cor:bounded-set}
        For a bounded set $K \subset \mathbb{R}^d$, $\bX$ absolutely continuous:
        \begin{align*}
            \bbP(\sup_{z\in K}|\SphereD_s(z|X_n)-\SphereD_s(z|\bX)|>R_{P_X,n}(\mathcal{F}_z)+t+2L\varepsilon)\\
            \leq N(\varepsilon,K,||\cdot||_2)2 e^{-\frac{nt^2}{2}} 
        \end{align*}
        where $L$ is the constant of Lipschitzness of Proposition~\ref{prop:lipschitz-depth}.
    \end{corollary}    

    In particular if $\bX$ has a bounded support, the consistency is obtained for all the support, and the last result of \eqnref{eq:consistency_support} follows. 

\section{FAST OPTIMIZATION SOLVER}
\label{s:algorithm}

     In practice, given samples $x_1, \dots, x_n$, the proposed depth of \eqnref{eq:sigmoid_sphere_depth} amounts to:
    \begin{equation}
        \label{eq:empirical_sphere_depth}
        \SphereD_s(z | \samples) = \inf_{c\in \bS(z, r)}  \frac{1}{n} \sum_{i=1}^n sig_s\left(r^2 - \|x_i - c\|^2\right)\enspace,
     \end{equation}
     which up to a change of variable, can be written as:
    \begin{equation}
        \label{eq:empirical_sphere_depth_u}
        \SphereD_s(z | \samples) = \inf_{\|u\|=1}  \frac{1}{n} \sum_{i=1}^n sig_s\left(r^2 - \|x_i - z - ru\|^2\right)\enspace,
     \end{equation}
We propose to solve \eqref{eq:empirical_sphere_depth_u} through Riemannian gradient descent on the unit sphere. 

\paragraph{Riemannian gradient descent}
Let $\cL: u \mapsto \sum_{i=1} sig_s\left(r^2 - \|x_i - z - ru\|^2\right)$ denote the loss function. Let $u \in \sphere(0, 1)$, let $\cT_u$ denote the tangent space of $\sphere(0, 1)$ at $u$ and $\pi_{\cT_u}$ the orthogonal projection on $\cT_u$. To compute a descent direction in $\cT_u$ we compute the tangent component of the gradient which is formally given by: $\pi_{\cT_u}(\nabla \cL(u)) = \nabla \cL(u) - \langle \nabla \cL(u), u\rangle u$. To perform a descent step with step size $\alpha$ on the sphere towards the descent direction $v\eqdef -\frac{\pi_{\cT_u}(\nabla \cL(u))}{\|\pi_{\cT_u}(\nabla \cL(u))\|}$, we follow the geodesic between $u$ and $v$ which is given by the exponential map:
\begin{equation}
    \label{eq:exp_map}
    \exp_{u, v}: \alpha \in [0, \pi] \mapsto \cos(\alpha)u + \sin(\alpha) v
\end{equation}

Setting a proper step-size $\alpha$ can be difficult in practice. To tackle this issue, we initialize $\alpha = \pi$ and decrease it for every step where $\cL(u)$ does not decrease. Moreover, due to the non-convex nature of the problem, the output of the optimization solver depends on the initialization. In practice, we noticed that initializing $u$ to the normalized mean of the samples $\samples$ provides a stable behavior of the solver and does not require any further tuning. The full procedure is described in Algorithm~\ref{alg:grad-ssd}.

    \begin{algorithm}[hbt!]
    \caption{Riemannian gradient descent}\label{alg:grad-ssd}
    \KwIn{$z,\samples, tol, \alpha $}
    \KwResult{$d\eqdef \SphereD_s(z|\samples)$} \;
    Initialize $u = \frac{1}{n}\sum_{i=1}^n x_i $ 
    \;\\ 
    $u \gets u/\|u\|$ \\
    $d \gets \cL(u) $\;\\
    \For{$i=1 \, \emph{\KwTo}$ \, $n_{iter}$}{
        \;$ v \gets -\nabla_u \cL(u)$\; \\
        \;$ v \gets v - \langle u,v\rangle u $\; \\
        $v \gets v/\|v\|$ \\
        \;$ u \gets \cos(\alpha)u + \sin(\alpha) v $ \; \\
        \;$d' \gets \cL(u) $ \; \\
        \;$dist = |d'-d| $ \; \\
         \lIf{$d' > d$}{$\alpha \gets \alpha / 2$}\;
         \lElseIf{$dist<tol$}{break}
   \;
        \lElse{$d \gets d'$}{}

    }
    \Return $d$
    \end{algorithm}

    \paragraph{Scalability for large $n$}
    To evaluate the computation speed-up gained by relaxing the original non-smooth optimization problem via the sigmoid function, we simulate a random centered and standard multivariate Gaussian in $\bbR^3$ and compute the depth of $z = [10, 10, 10]$ for both the halfspace and sphere depth. For HD, we use the zero-order Nelder-Mead algorithm. For $\SphereD_s$ we use Algorithm \ref{alg:grad-ssd}. Figure \ref{fig:speed} shows how the propose manifold gradient descent on the sphere is significantly more scalable.

    \begin{figure}[hbt!]
    \includegraphics[width=\linewidth]{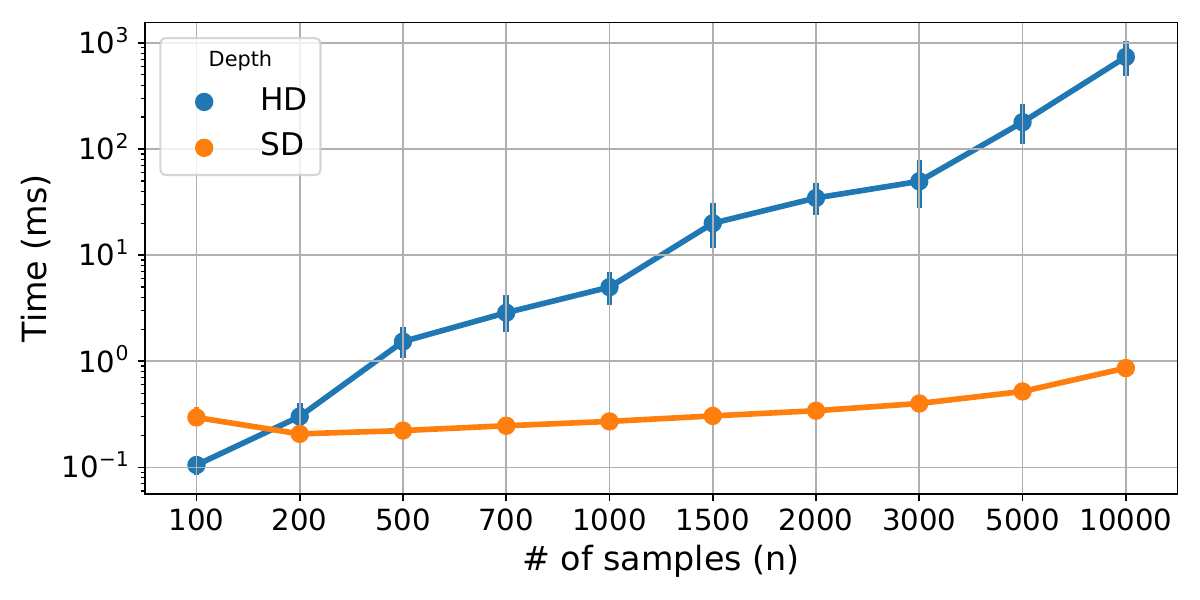}
    \caption{Illustration of the scalability of the proposed $\SphereD_s$ computed using Alg. \ref{alg:grad-ssd} compared to gradient-free optimization based HD\label{fig:speed}.}
    \end{figure}
main
\section{APPLICATIONS AND NUMERICAL EXPERIMENTS}

    \subsection{Simulated multimodal distributions}\label{subsec:simulated-ranks}
    

        To confirm the order-correctness of the proposed depth notion, we conducted a number of rank-based comparisons with the true density. 

         To test the handling of multimodality, we carried out 50 independent simulations with a distribution of two Gaussians $\mathcal{N}([-3.5,-3.5],I_d)$ and $\mathcal{N}([3.5,3.5],I_d)$, computing the rank correlations between the ranks of the data points obtained by depths computation and the ranks by the true density, and it shows that the Sphere depth's ordering was closer to the one of the true density than LSPD (localised spatial depth). For LSPD, we set the kernel bandwidth parameter $h=1$ to be at the crossing of its two regimes, and for $\SphereD_s$ we used $r=1,s=1$. 
         We show the impact of the dimension for a range $d=2,4,6,8$, for $n=200$ samples in Figure~\ref{fig:dgrowing} using Spearman rank correlation. We also illustrate the consistence by the convergence of the ordering as $n$ grows from $100$ to $1000$ for 
         $d=10$ in Figure~\ref{fig:ngrowing}. In section~\ref{app:exp} of the Appendix, similar graphs for the Kendall tau rank correlation are displayed.
                \begin{figure}[h]
        \centering\includegraphics[width=\linewidth]{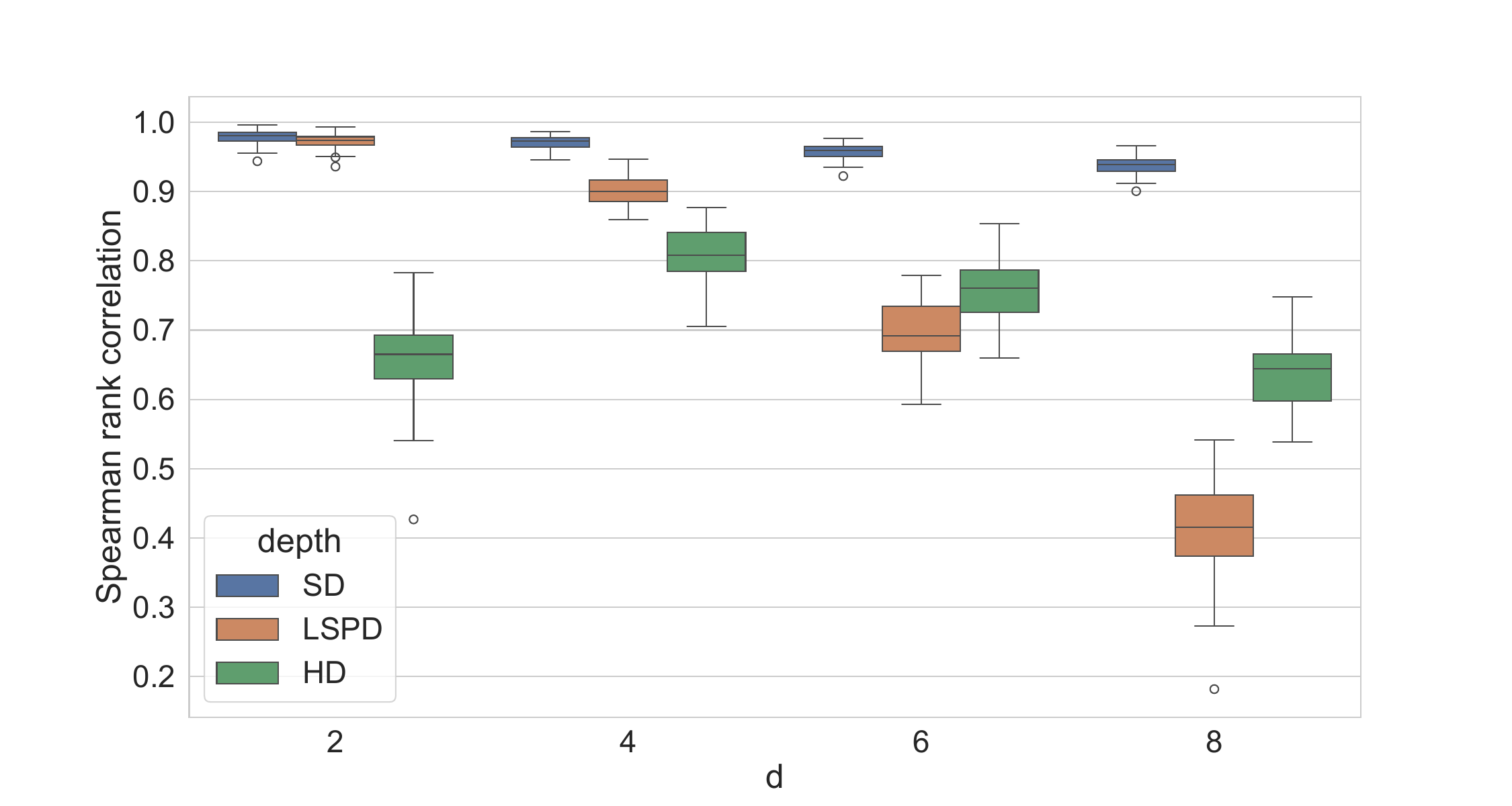}
                         \caption{Bi-Gaussian $SD^1_1$ ($r,s=1$) vs LSPD ($h=1$) Spearman correlation ranking w.r.t. true density (50 runs)}
                     \label{fig:dgrowing}
                \end{figure}

                \begin{figure}[h]
        \centering\includegraphics[width=\linewidth]{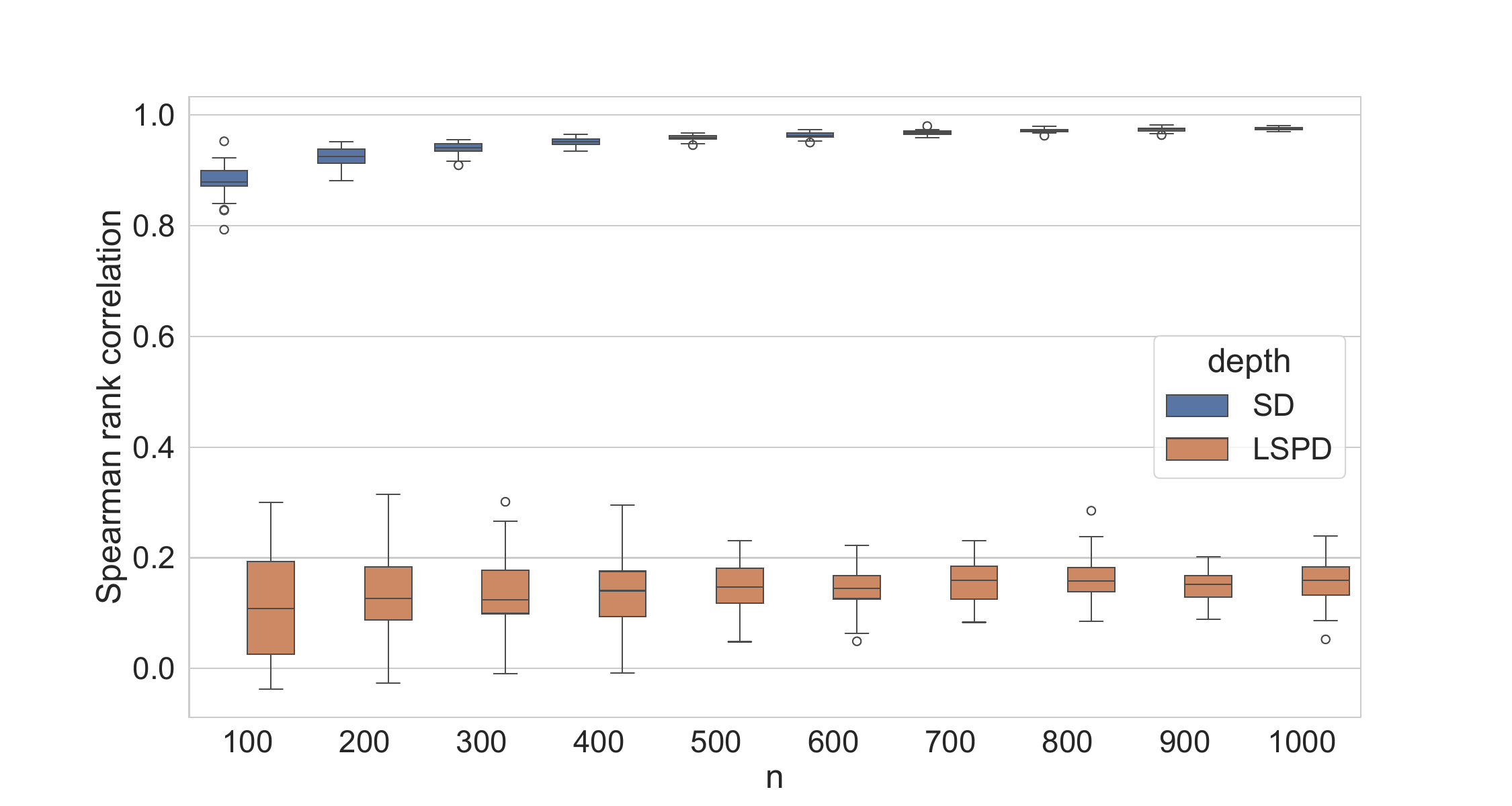}
                         \caption{Convergence of the Spearman rank correlation w.r.t. the true density for a bi-Gaussian distribution d = 10}
                         \label{fig:ngrowing}
                \end{figure}

    \subsection{Application to homogeneity tests}\label{subsec:homogeneity-tests}

Data depth can serve to construct a homogeneity test. Indeed, \citet{liu1993quality} introduced a quality index $Q$ defined with respect to a data depth $D$ and some probability distributions $\bF$, $\bG$ as:
    \begin{equation}
        Q(\bF,\bG)=\bbP(D(X|\bF)\leq D(Y|\bF)|X\sim~\bF, Y\sim~\bG).
    \end{equation}
 
    In the case $\bF=\bG$, for non-trivial cases $Q(\bF,\bG)=1/2$: indeed, since the depth value distributions are the same, picking two values at random there will be probability one half that one is less than the other (assuming the data depth and the distributions considered allow to exclude equality situations with non-negligible density). This allows under suitable conditions to use the quality index for homogeneity testing.

    \citet{zuo2006limiting} exhibits some set of such conditions that guarantees the difference between the empirical quality index $Q(\bF_n,\bG_m)$ and the population quality index to converges asymptotically to a normal distribution (after rescaling by some term of order $\Theta(\sqrt{n+m})$).

    Here, we assume we are under the null hypothesis with two sets of independent samples $X_{1:n}$ and $Y_{1:m}$ coming from the same distribution $\bF$ and only prove the asymptotic convergence of the quality index in that case in order to get the asymptotic behaviour of testing against the null hypothesis.
    Since the quality index is only computed on the elements coming from the distribution, and we assume we have only one distribution to consider under the null hypothesis, we relax the assumptions of the work of \citet{zuo2006limiting}  (in particular by taking suprema on the support of $\bF$ instead of all $\mathbb{R}^d$).

    We prove a result similar of that of \citet{zuo2006limiting} in the case of the null hypothesis for a probability distribution $\bX$ and a depth $D$ verifying the following:
    \begin{itemize}
        \item (A1) $\bbP(d_1\leq D(Z|\bX)\leq d_2)\leq C|d_2-d_1|$ for some $C$ and any $d_1,d_2\in [0,1], Z\sim \bX$
        \item (A2) $\sup_{z\in supp(\bX)} |D(z|\samples)-D(z|\bX)|=o(1)$ almost surely as $n\to\infty$
        \item (A3) $\mathbb{E}(\sup_{supp(\bP_{\bX})}|D(z|\samples)-D(z|\bX)|)= O\left( \frac{log(n)}{\sqrt{n}} \right)$,
    \end{itemize}
    see Appendix, Section~\ref{app:proofs}.
    In the work of \citet{zuo2006limiting}, for most depths, (A1) is generally assumed.

    We show that in the case of the Sphere Sigmoid depth, 
    (A2) and (A3) are fulfilled for absolutely continuous distribution $X$ with bounded support.

    Indeed, for such distribution, (A2) is proved by Theorem~\ref{thm:main}.
    We prove also that (A3) is fulfilled under the same hypotheses:

    \begin{proposition}\label{prop:A3}
        For $\bX$ absolutely continuous with bounded support:
        \begin{equation*}
            \mathbb{E}(\sup_{supp(\bX)}|\SphereD_s(z|\samples)-\SphereD_s(z|\bX))|)= O\left( \frac{\log(n)}{\sqrt{n}} \right)
        \end{equation*}
    \end{proposition}
    The proof is provided in the Appendix, Section\ref{app:proofs}.

    With this, we get a more practical result for the proposed sphere depth $\SphereD_s$:
    \begin{theorem}\label{th:homogeneity-test}
        Let there be two sets of $n$ and $m$ independent samples respectively coming from a distribution $\bF$ absolutely continuous with bounded support and verifying (A1), each sets of samples with empirical distribution $\bF_n$ and $\bF'_m$ respectively. Then the quality index using $\SphereD_s$ verifies:
        \begin{equation*}
            \left(\frac{1}{12}(\frac{1}{n}+\frac{1}{m})\right)^{-1/2}(Q(\bF_n,\bF'_m)-Q(\bF,\bF))\to \mathcal{N}(0,1)
        \end{equation*}
         in distribution as the number of samples $n,m$ goes to infinity.
    \end{theorem}

    Therefore one can design a test for the null hypothesis that two sets of samples comes from the same distribution by computing the quality index with $\SphereD_s$ and comparing it to $1/2$, since for a reasonably high number of samples, one can make use of the quantiles of the normal distribution to design a threshold for the test.

        
        We carry out the same experiment of \citet{shi2023two} albeit we use truncated distributions to be able to use our theoretical results. The distributions are all in two dimensions. 
        The first distribution is multivariate t distribution with degrees of freedom 2, mean $[0,0]$, scale matrix $I_2$, while the second is one with degrees of freedom 3, mean $[0,0]$, scale matrix $I_2+0.6\tilde{I}_2$, where $\tilde{I}_2= [[0,1],[1,0]]$; the samples with norm higher than 10000 are truncated. We draw 1000 times two sets of $n=100,200,...,500$ samples. We display both $Q(\bF,\bG)$ and $Q(\bG,\bF)$ 
        for Mahalanobis depth (MD) and the Sphere depth (SD) on Figure~\ref{fig:htest}.


                \begin{figure}[h]
                     \centering
                     \includegraphics[width=\linewidth]{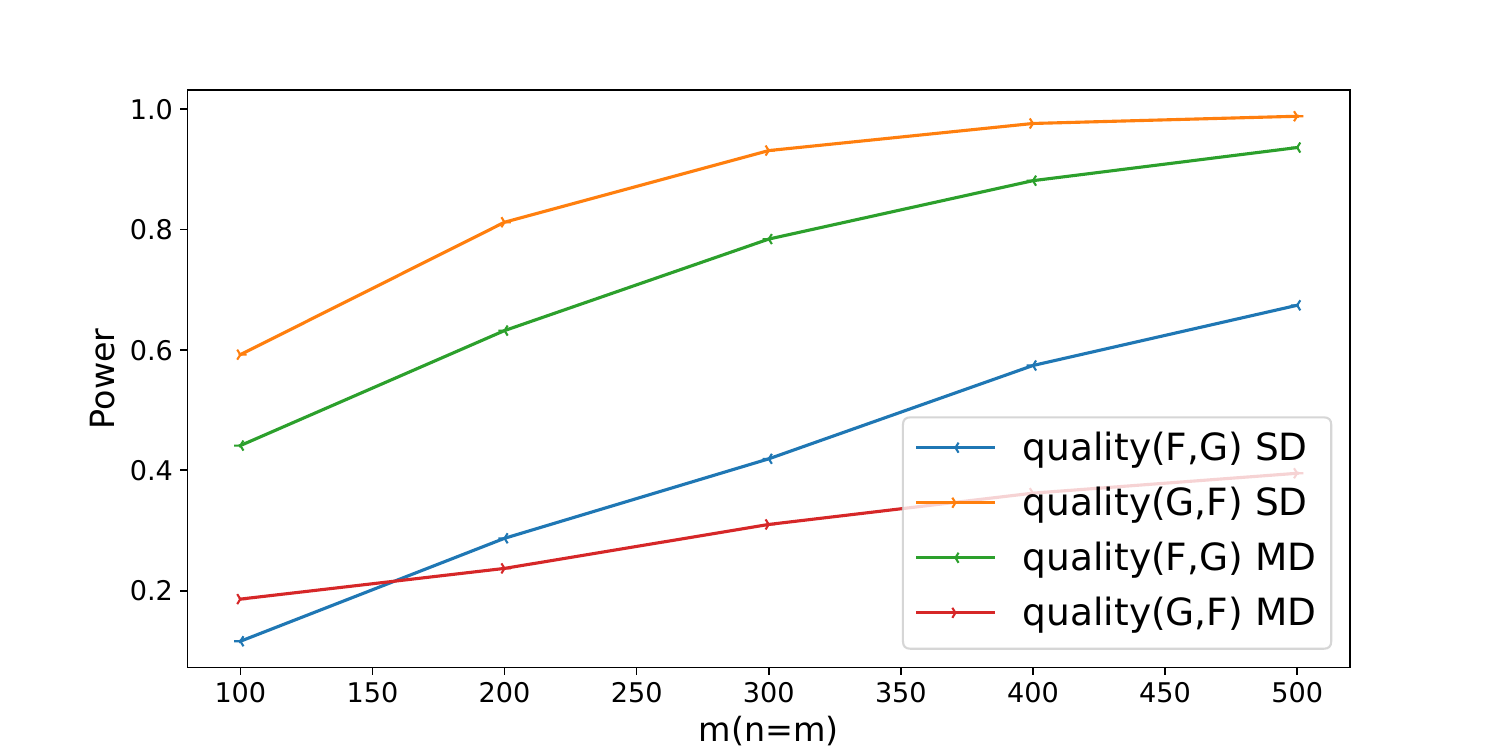}
                         \caption{Power for homogeneity test between different multivariate t distribution truncated at 10000, 1000 repetitions}
                         \label{fig:teststudent}
                \label{fig:htest}
                \end{figure}

    \subsection{Anomaly detection on real data}\label{subsec:anomaly}
    Using several real datasets different datasets of the ODDS database ~\citep{Rayana:2016}, we perform unsupervised anomaly detection by computing an outlier score proportional to the probability of a data point being an outlier. For data depth functions, such a score can be given by 1 - depth.  Using this score and the ground truth labels, we compute the Area Under the Receiving Operator Curve (AUROC or AUC-ROC).  Results are displayed in Table~\ref{tab:auroc}.
    We also compute the h-measure~\citep{hand2009measuring,hand2014better} of the scores in the Appendix, section~\ref{app:exp}.
 We compare the performance of the proposed depth to that of halfspace-depth (HD), one class SVM (OCSVM), Localized spatial depth (LSPD) and Local outlier factor (LOF)~\citep{breunig2000lof}. The idea behind (LOF) is to compute a local ``reachability" term defined using the distance between the points and its $k$ neighbours, and then compare this quantity with the ones of its neighbours to see if it differs or not. For the parameter $s$ of $\SphereD_s$, we use the standard deviation of the data multiplied by the dimension $d$ and a radius $r$ set to the standard deviation of the data. For all the competitors however, we display their best performance over a grid of hyperparameters. For LOF, we select the best $k \in [5, 10, 15, 20, 30]$. For kernel based methods (OC-SVM and LSPD) we used a Gaussian kernel with a hyperparameter $\gamma$ in a logarithmic scale grid of 20 values within $[0.001, 1000] * \frac{1}{\bbV(X) d}$ (and similarly for parameter $h$ in LSPD). We set the contamination rate $\nu$ for OC-SVM and LOF to their default value. 

Table~\ref{tab:auroc} shows that even though the hyperparameters $s$ and $r$ were not optimized, the proposed depth $\SphereD_s$ still manages to be on top for 5 datasets while remanining competitive on the others.

\begin{table}[!htb]
\centering
\caption{AUC-ROC on anomaly detection datasets}\label{tab:auroc} 
\begin{tabular}{lrrrrr}
\toprule
 & OCSVM & LSPD & LOF & HD & SD  \\
 \midrule
wine & 0.80 & 0.58 & 0.94 & 0.50 & \textbf{0.96}  \\
glass & \textbf{0.87} & 0.50 & 0.78 & 0.61 & 0.77  \\
vertebral & \textbf{0.70} & 0.60 & 0.48 & 0.52 & 0.50  \\
vowels & 0.83 & 0.82 & \textbf{0.97} & 0.56 & 0.75  \\
pima & 0.61 & 0.66 & \textbf{0.67} & 0.58 & 0.66  \\
breastw & \textbf{0.94} & 0.82 & 0.59 & 0.86 & \textbf{0.94}  \\
lympho & 0.87 & 0.50 & \textbf{0.89} & 0.50 & 0.87 \\
thyroid & \textbf{0.98} & 0.98 & 0.77 & 0.86 & \textbf{0.98}  \\
annthyroid & \textbf{0.86} & 0.85 & 0.51 & 0.61 & 0.79 \\
pendigits & 0.79 & 0.59 & 0.54 & 0.59 & \textbf{0.81} \\
cardio & 0.80 & 0.50 & 0.39 & 0.36 & \textbf{0.84} \\
\bottomrule
\end{tabular}
\end{table}

\section{Conclusion}
Inspired by kernel methods, we introduced the sphere depth, a new data depth for tackling multimodal distributions, which is the first data depth, to our knowledge, that makes proper use of the gradient descent algorithm, achieving fast optimization. We proved that our data depth verifies some general depth properties, as well as consistency and concentration bounds, that allow us to use it in homogeneity testing. Simulations confirm the geometric intuition behind the depth that multimodal distributions can be handled, and experiment on real data shows a very competitive performance on anomaly detection. 

\setlength{\itemindent}{-\leftmargin}
\makeatletter\renewcommand{\@biblabel}[1]{}\makeatother
\bibliographystyle{apalike}
\bibliography{arxiv_ref}

\onecolumn
\appendix 

\section{REMINDER}\label{app:reminder}
\subsection{Data depth}
Given a probability distribution $\bX$ in $\cP(\bbR^d)$ with the corresponding random vector $X$ distributed as $\bX$ (i.e., $X \sim \bX$), and an arbitrary point $z\in\bbR^d$, statistical data depth function constitutes the following mapping
\begin{equation*}
    D\,:\,\bbR^d \times \cP(\bbR^d) \rightarrow \bbR\,,\,(z|\bX)\,\mapsto\,[0,1]\,,
\end{equation*}
which satisfies the conditions of:
\begin{itemize}
    \item \emph{Affine invariance}: $D(A z + b|A X + b) = D(z|\bX)$ for any $b\in\bbR^d$ and any non-degenerate $d \times d$ matrix $A$.
    \item \emph{Monotonicity}: For any $z^*\in\bbR^d$ having maximal depth, i.e., $\bbR^d\ni z^*\in\argmax_{ z\in\bbR^d} D(z|\bX)$, for any $r$ such that $\|r\|=1$, the function $\beta\mapsto D(z^* + \beta r|\bX)$ is non-increasing for $\beta>0$.
    \item \emph{Vanishing at infinity}: $\lim_{\|z\| \rightarrow\infty}=0$.
    \item \emph{Upper-semicontinuity}: The depth-trimmed regions, i.e., upper-level sets of the depth function $D_\alpha(\bX) = \{z\in\bbR^d\,:\,D(z|\bX)\ge\alpha\}$ are closed for all $\alpha\in[0,1]$.
\end{itemize}
This set of conditions, often called postulates, being taken from (one of) the (most) recent survey(s) by \cite{MoslerM22}, takes its origins in~\cite{dyckerhoff2004data} and~\cite{mosler2013depth}. A slightly different, but equivalent, set of postulates can be found in~\cite{ZuoS00a}. The reader is further referred to the surveys by~\cite{LiuPS99} and~\cite{Cascos09} for more details on the nature of statistical data depth function, as well as referenced works therein.

\subsection{Kernel methods}
    Kernel methods \citep{scholkopf2002learning} allow to tackle nonlinear relationship between data by leveraging linear tools in well chosen infinite dimensional spaces. They rely on the key notion of Reproducing Kernel Hilbert Space \citep{aronszajn1950theory, berlinet2011reproducing} and the famous kernel trick. For a positive definite symmetric kernel $k: \cX \times \cX \to \bbR$ defined over some 
    non-empty set $\cX$ (here, in this work, $\cX =\bbR^d$), Aronszajn's theorem says that there is a unique Hilbert space $\cH$ equipped with an inner product $\langle \cdot, \cdot \rangle_{\cH}$, for which $k(\cdot, x) \in \cH$ and $k$ is its reproducing kernel, that is:
   \begin{equation*}
       \forall f \in \mathcal{H},\forall x \in X, \langle f, k(\cdot,x)\rangle_\mathcal{H} = f(x).
   \end{equation*}
   In the remainder of the paper, we denote $\varphi: \cX \to \cH$ the canonical feature map defined by: $\varphi(x):=k(\cdot,x) \in \cH$.
   This defines what is called a kernel inner product on the space $X$:
       \begin{equation*}
           k(x,y) = \langle \varphi(x), \varphi(y) \rangle_\mathcal{H}
       \end{equation*}
    In practice, an explicit formula often allows to compute the kernel inner product without going through the infinite dimensional RKHS - this is called the {\it kernel trick}.
\section{PROOFS}\label{app:proofs}

\subsection{Proofs for section~\ref{subsec:depth-geom-prop}}

\newtheorem*{prop:depth-geom-properties}{Proposition \ref{prop:depth-geom-properties}}
\begin{prop:depth-geom-properties}
        Let $z\in\bbR^d$ and $\bX$ a probability distribution in  $\cP(\bbR^d)$ and $X \sim \bX$. For any $s \geq 0$:
        \begin{itemize}
            \item (i)  $\SphereD_s(z|\bX)\in [0,1]$ and vanishes at infinity:\\ $\SphereD_s(z|\bX)\to 0$ when $||z||\rightarrow +\infty$,
            \item (ii) $\SphereD_s(z|\bX)$ is invariant by any global isometry $I$ of $\bbR^d$: $\SphereD_s(I(z)|I(\bX))=\SphereD_s(z|\bX)$.
        \end{itemize}
\end{prop:depth-geom-properties}

    \begin{proof}

        \begin{itemize}
            \item (i): $sig_s$ has values in $[0,1]$ even when $s \to 0$, therefore the depth is in $[0,1]$.
            For any $\varepsilon$, pick some $t$ big enough such that $P_X(||x||\leq t)\geq 1-\varepsilon$. Then for some big constant $M$, assume $z$ verifies $||z||>t+r+M$, we have in particular for any $c\in \sphere(z,r)$ that $||c||>t+M$. Therefore by the triangular inequality, for any $x$ such that $||x||\leq t$, $||x-c||>M$. If we define $\eta = sig_s(r^2-M^2)$, then $sig_s(r^2-||x-c||^2)<\eta$.  Therefore, we have :
            \begin{equation*}
                 \bbE \left[ sig_s\left(r^2 - \|X - c\|^2\right)\right] < (1-\varepsilon)\eta + \varepsilon.
            \end{equation*}
            As this is true for any $c\in \sphere(z,r)$, by taking the infinum $\SphereD_s (z|\bX) < (1-\varepsilon)\eta + \varepsilon$. When $||z|| \to \infty$, we can take $t$ and $M$, arbitrarly big, therefore $\eta$ and $\varepsilon$ can be arbitrarly small and so, since the depth is positive $\SphereD_s (z|\bX) \to 0$.
            \item (iii): Let $I$ be some global isometry of $\mathbb{R}^d$. We will first prove that
            $\SphereD_s (I(z)|I(\bX)) \leq \SphereD_s (z|\bX)$. For any $c' \in \sphere(I(z),r), \exists!c\in \sphere(z,r), I(c)=c'$ by bijection and isometry. Since $||I(x)-c'||=||x-c||$ by isometry, noting $f_{z,c}:x\to sig_s(r^2-||x-c||^2)$ as in section~\ref{sec:properties}, we get $f_{z,c'}(I(x))=f_{z,c}(x)$.
            
            Optimising over all possible $c'$ passing to the infinum gives $\SphereD_s (I(z)|I(\bX)) \leq \SphereD_s (z|\bX)$. By applying the same reasoning with $I^{-1}$, we prove the equality $\SphereD_s (I(z)|I(\bX)) = \SphereD_s (z|\bX)$.
        \end{itemize}\qed
    \end{proof}

\newtheorem*{prop:scaling}{Proposition \ref{prop:scaling}}
    \begin{prop:scaling}
    For any $z\in\bbR^d$ and $\bX$ a probability distribution, $\lambda, s>0$:
    \begin{align*}
        SD^{\lambda r}_0 (\lambda z | \lambda X) = \SphereD_{0} (z|\bX)\\
        SD^{\lambda r}_s (\lambda z | \lambda X) = SD^{r}_{\frac{s}{\lambda^2}} (z|\bX)
    \end{align*}
    \end{prop:scaling}

    \begin{proof}
         We emphasised in our statement that for $s=0$, only $r$ has to be changed as we used an indicator function on the distances, only the inequality matters rather than the values of the distances. We prove now an equivalent formulation, for $s\geq0,r>0$:
         \begin{equation*}
             SD^{\lambda r}_{\lambda^2s} (\lambda z | \lambda X) = SD^{r}_{s} (z|\bX)
         \end{equation*}

         Remark that $S(\lambda z,\lambda r)=\lambda S(z,r)$. Noting $c'=\lambda c,x'=\lambda x$ for some $c\in S(z,r)$, $x \in \mathbb{R}^d$:
         \begin{align*}
            sig_{\lambda^2s}((\lambda r)^2 - ||x'-c'||) &= sig_{\lambda^2s}(\lambda^2 (r^2- ||x-c||^2))\\
                                            &= \frac{1}{1+e-\frac{\lambda^2(r^2- ||x-c||^2)}{\lambda^2s}}\\
                                            &= sig_s(\lambda^2 (r^2- ||x-c||^2))
        \end{align*}
    We conclude by passing to the infinum over all possible $c$.\qed
    \end{proof}

\newtheorem*{prop:SD-less-than-HD}{Proposition \ref{prop:SD-less-than-HD}}
    \begin{prop:SD-less-than-HD}
        For any $z\in\bbR^d$ and any $r > 0$:
        \begin{equation*}
            \SphereD_0 (z|\bX) \leq \TukeyD (z|\bX)
        \end{equation*}
    \end{prop:SD-less-than-HD}

    \begin{proof}
        We use the characterisation of the depth $\SphereD_0 (z|\bX)$ by looking at the density inside of a ball of \eqnref{eq:ball-form} by Proposition~\ref{prop:sphere}. 
        For some $u$ belonging to the unit sphere, let $H_{u,z}$ be some halfplane defined by the equation:
        \begin{equation}\label{eq:halfplane-u}
            \langle x, u\rangle - \langle z,u \rangle \geq 0,
        \end{equation}
        whose hyperplane therefore contains $z$ and is orthogonal to $u$.
        There exists $c \in \sphere(z,r)$ verifiying the equation $c=z+ru$, 
        and by multiplying \eqnref{eq:halfplane-u} by $r>0$, we get the equivalent equation:
        \begin{equation}\label{eq:halfplane-c}
            \langle x, c-z\rangle - \langle z,c-z \rangle \geq 0,
        \end{equation}
        We prove that for $x \in \bbB(c,r)$, $x$ verifies \eqnref{eq:halfplane-c}. Without loss of generality, we translate everything by $-z$ to assume $z=0$. Then \eqnref{eq:halfplane-c} becomes:
        \begin{equation*}
            \langle x, c\rangle \geq 0
        \end{equation*}
        and $x \in \bbB(c,r)$ is equivalent to:
        \begin{align*}
            &||c-x||^2\leq||c-z||^2=||c||^2 \\
            \iff &||c||^2+||x||^2-2\langle x,c \rangle \leq ||c||^2 \\
            \iff &||x||^2 \leq 2\langle x,c \rangle
        \end{align*}
        therefore $x$ must belong to $H_{u,z}$: 
        since $\ball(c,r)\subseteq H_{u,z}$, $\mathbb P_\bX(\bbB(c,r))\leq \mathbb P_\bX(H_{z,c})$, we get the result by passing to the infinum. \qed
    \end{proof}

\subsection{Proofs for section~\ref{sec:properties}}

\newtheorem*{prop:lip}{Proposition \ref{prop:lip}}
    \begin{prop:lip}
        The functions $f_{z,c}$ of $\mathcal{F}_z$ are Lipschitz (with respect to $x$).
        Symmetrically, they are also Lipschitz with respect to $c$.
    \end{prop:lip}

    \begin{proof}
        We do so by bounding the gradient of $f_{z,c}$
        \begin{align}
            \nabla_x f_{z,c}(x) = \nabla_x(r^2-||c-x||^2) sig_s'(r^2-||c-x||^2) \label{eq:nabla}.
        \end{align}
        Note that the derivative of the sigmoid is:
        \begin{equation*}
            sig_s'(x)=\frac{1}{s} sig_1(x) (1-sig_1(x)),
        \end{equation*}
        so \eqnref{eq:nabla} becomes:
        \begin{equation*}
            \nabla_x f_{z,c}(x) = \frac{\nabla_x(r^2-||c-x||^2)}{s} sig_1(r^2-||c-x||^2) (1-sig_1(r^2-||c-x||^2)).
        \end{equation*}
        We have $\nabla_x(r^2-||c-x||^2) = 2(x-c)$, therefore:
            \begin{align*}
                ||\nabla_x f_{z,c}(x)|| = \frac{2||x-c||}{s} sig_1(r^2-||c-x||^2) (1-sig_1(r^2-||c-x||^2)).
            \end{align*}
        This expression is bounded by some constant $L$ as it is continuous and for $||x-c||\to\infty,||\nabla_x f_{z,c}(x)||\to 0$, since $(1-sig_1(r^2-||c-x||^2)=\frac{e^{-(r^2-||c-x||^2)}}{1+{e^{-(r^2-||c-x||^2)}}}$ and the exponential decrease dominates. Therefore $f_{z,c}$ is L-Lipschitz, where $L=\sup_{x\in\bbR^d} ||\nabla_x f_{z,c}(x)||$. By symmetry of the role of $c$ and $x$ that only appears in the distance $||c-x||$, by swapping the roles of the two, the same Lipschitzness is true with respect to $c$.\qed
    \end{proof}

Proposition~\ref{prop:bracketnum} makes use of Propositon~\ref{prop:lip}, as well as of Lemma~\ref{lem:ballcover} that we fully restate here:

\newtheorem*{lem:ballcover}{Lemma \ref{lem:ballcover}}
    \begin{lem:ballcover}
        The covering number of the Euclidian unit ball $B$ in dimension d is of order:
        $N(\varepsilon,B,||\cdot||_2) = \Theta\left(\frac{1}{\varepsilon^d}\right)$ for $0<\varepsilon <1$ (for $\varepsilon\geq 1$, one element is sufficient to cover the ball).
    \end{lem:ballcover}
    
In \citet{vershynin2018high}(Corollary 4.2.11) in particular, it is stated that, for any $\varepsilon>0$:
\begin{equation}\label{eq:vershynin-ball}
    \left(\frac{1}{\varepsilon}\right)^d \leq N(\varepsilon,B,||\cdot||_2)  \leq \left(\frac{2}{\varepsilon}+1\right)^d .
\end{equation}
By dilating everything by $r>0$, we get that $N(\varepsilon r,\ball(0,r),||\cdot||_2)=N(\varepsilon,\ball(0,1),||\cdot||_2)$ or in other words $N(\varepsilon,\ball(0,r),||\cdot||_2)=N(\varepsilon/r,\ball(0,1),||\cdot||_2)$.

Before proving Proposition~\ref{prop:bracketnum}, we will also need this theorem:
\begin{theorem}~\cite[Th 2.7.11 reformulated]{van1996weak}\label{thm:van-lip}

   For $\mathcal{F}=\{f_c:c\in C\}$ defined for a metric space $(C,d)$, and $F$ some (envelope) function, if:
    \begin{align*}
        \forall f_c,f_{c'}\in \cF, \forall x, |f_c(x)-f_{c'}(x)| \leq d(c,c')F(x)
    \end{align*}
    then, for any norm $||\cdot||$:
    \begin{align*}
        N_{[]}(2||F||\varepsilon,\mathcal{F},||\cdot||)\leq N(\varepsilon,C,d).
    \end{align*}
\end{theorem}

Now we can prove:
\newtheorem*{prop:bracketnum}{Proposition \ref{prop:bracketnum}}
    \begin{prop:bracketnum}
    The bracketing number of $\cF_z$ is bounded: $N_{[]}(\varepsilon,\mathcal{F}_z,L_1(P)) <\infty$ for any $\varepsilon>0$. In particular, for $\varepsilon$ sufficiently small, 
        $N_{[]}(\varepsilon,\mathcal{F}_z,L_1(P)) \leq O\left((\frac{r}{\varepsilon})^d\right)$, otherwise 
        $N_{[]}(\varepsilon,\mathcal{F}_z,L_1(P)) = O(1)$.
    \end{prop:bracketnum}

    \begin{proof}
    By Proposition~\ref{prop:lip}, any $f_{z,c}\in\cF_z$ is $L$-Lipschitz, therefore we can apply Theorem~\ref{thm:van-lip} taking the constant function $F:x\to L$ as envelope to relate the bracketing number of $\cF_z$ to the covering number of $\sphere(z,r)$:
        \begin{align*}
            N_{[]}(\varepsilon,\mathcal{F}_z,L_1(P))&\leq N(\varepsilon/2L,\sphere(z,r),||.||_1)\\
            &\leq N(\varepsilon/2L,\sphere(z,r),||.||_2)\\
            &\leq N(\varepsilon/2L,\ball(0,r),||.||_2)\\
            &= N(\varepsilon/2Lr,\ball(0,1),||.||_2)\\.
        \end{align*}
    We conclude using Lemma~\ref{lem:ballcover} that 
    $N_{[]}(\varepsilon,\mathcal{F}_z,L_1(P))= \Theta \left((\frac{r}{\varepsilon})^d\right)$ for $0<\varepsilon<1$ and $N_{[]}(\varepsilon,\mathcal{F}_z,L_1(P))= \Theta (1)$ for higher $\varepsilon$.\qed
    \end{proof}

Now from this result follows :

\newtheorem*{prop:rademacher-fz}{Proposition \ref{prop:rademacher-fz}}
    \begin{prop:rademacher-fz}
    The Rademacher complexity of $\cF_z$ with respect to some distribution $P$ is of order:
    \begin{equation*}
        R_{P,n}(\mathcal{F}_z) = O\left(\sqrt{\frac{d}{n}}\right)
    \end{equation*}
    \end{prop:rademacher-fz}

    \begin{proof}
    By Theorem~\ref{thm:dudley} (Dudley), we get a bound using covering numbers.
    The same reasoning as previous using Theorem~\ref{thm:van-lip} is valid for $L_2$ and enables to relate to the covering of the ball.
    \begin{align*}
    R_{P,n}(\mathcal{F}_z)
        &\leq \int_0^{\infty} \sqrt{\frac{\log N(\varepsilon,\mathcal{F}_z,L_2(P))}{n}}d\varepsilon\\
        &\leq \int_0^{\infty} \sqrt{\frac{\log N_{[]}(2\varepsilon,\mathcal{F}_z,L_2(P))}{n}}d\varepsilon \text{  (cf. \citep{van1996weak}  p.84)}\\
        & \leq \int_0^{\infty} \sqrt{\frac{\log N(2\varepsilon/Lr,\ball(0,1),||.||_1)}{n}}d\varepsilon \\
        & \leq \int_0^{\infty} \sqrt{\frac{\log N(2\varepsilon/Lr,\ball(0,1),||.||_2)}{n}}d\varepsilon \\
        &=  \int_0^{Lr/2} \sqrt{\frac{\log N(2\varepsilon/Lr,\ball(0,1),||.||_2)}{n}}d\varepsilon +  \int_{Lr/2}^{\infty} \sqrt{\frac{\log N(2\varepsilon/Lr,\ball(0,1),||.||_2)}{n}}d\varepsilon
    \end{align*}  
    The second integral where $2\varepsilon/Lr \geq 1$ leads to $0$ ($\log(1)$ since $N(2\varepsilon/Lr,\ball(0,1),||.||_2)=1$). Also, for the first integral, we know using eq.\ref{eq:vershynin-ball}, that for $0<\varepsilon'<1$:
    \begin{equation*}
        N(\varepsilon',B,||\cdot||_2)  \leq \left(\frac{3}{\varepsilon'}\right)^d.
    \end{equation*}
    Therefore: 
    \begin{align*}
        R_{P,n}(\mathcal{F}_z) &\leq \int_0^{Lr/2} \sqrt{\frac{\log (3Lr/2\varepsilon)^d}{n}}d\varepsilon\\
        &= O\left(\sqrt{\frac{d}{n}} \right)
    \end{align*}\qed
    \end{proof}

\newtheorem*{prop:lipschitz-depth}{Proposition \ref{prop:lipschitz-depth}}
    \begin{prop:lipschitz-depth}
        The Sphere depth $z \to \SphereD_s(z|\bX)$ is Lipschitz with respect to $z$.
    \end{prop:lipschitz-depth}  

    \begin{proof}
        Let $z\in\bbR^d$, any $c\in\sphere(z,r)$ can be rewritten $c=z+ru$ where $u\in\sphere(0,1)$. Consider another point $z'\in\bbR^d$, keeping the same $u$ translates into a point $c'\in\sphere(z',r)$ with $c'=z'+ru$, such that $||c'-c||=||z'-z||$. Writing $f_{z,c}:x\to sig(r^2-||z+ru-x||^2)$ we can define the function $g_u : (z,x) \to  sig(r^2-||z+ru-x||^2)$ and notice by the same idea as in the proof of Proposition~\ref{prop:lip}, 
        that $g_u$ is Lipschitz in $z$ for some constant $L'$.
        Therefore $|f_{z,c}(x)-f_{z',c'}(x)|=|g_u(z,x)-g_u(z',x)|\leq L' ||z-z'||$.
        So 
        \begin{equation*}
            \forall c\in\sphere(z,r)\exists c'\in\sphere(z',r), f_{z',c'}(x) \leq f_{z,c}(x) + L' ||z-z'||, 
        \end{equation*}         
        therefore by passing to the infinum and the expectation $\SphereD(z'|\bX) \leq \SphereD(z'|\bX) + L' ||z-z'||$, and vice-versa, so 
        $|\SphereD(z'|\bX) - \SphereD(z|\bX)|\leq L' ||z-z'||$ 
        and the Lipschitzness is proven.\qed
    \end{proof}

This proposition with Proposition~\ref{prop:bounded-rademacher}, that we restate here as well, give us Corollary~\ref{cor:bounded-set}.
\newtheorem*{prop:bounded-rademacher}{Proposition \ref{prop:bounded-rademacher}}
    \begin{prop:bounded-rademacher}
        For $\mathcal{F}$ a space of functions that are bounded by $b$ (in $||\cdot||_{\infty}$)
        \begin{equation*}
            \bbP(\sup_{f\in \mathcal{F}} \lvert \mathbb{E}_{\bX}[f(X)]- \mathbb{E}_{X_n}[f(X)] \rvert>\Radem_{\bX,n}(\mathcal{F})+t)\leq 2 e^{-\frac{nt^2}{2b^2}}
        \end{equation*}
    \end{prop:bounded-rademacher}

Proposition~\ref{prop:bounded-rademacher} is a classic result that can be proved using the fact that:
\begin{equation*}
    \bbE_{X_{1:n}} \left[ \sup_{f\in \mathcal{F}} \lvert \mathbb{E}_{\bX}[f(X)]- \frac{1}{n}\sum f(x_i)| \right] \leq 2 \Radem_{\bX,n}(\mathcal{F})
\end{equation*}
and that the function $\phi : (x_1,...,x_n) \to \sup_{f\in \mathcal{F}} \lvert \mathbb{E}_{\bX}[f(X)]- \frac{1}{n}\sum f(x_i)|$ changes by at most $b/n$ when changing one $x_i$, therefore the result can be concluded using McDiarmid's inequality~\citep{mcdiarmid1989method}.

\newtheorem*{cor:bounded-set}{Corollary \ref{cor:bounded-set}}
    \begin{cor:bounded-set}
        For a bounded set $K \subset \mathbb{R}^d$, $\bX$ absolutely continuous:
        \begin{align*}
            \bbP(\sup_{z\in K}|\SphereD_s(z|X_n)-\SphereD_s(z|\bX)|>R_{P_X,n}(\mathcal{F}_z)+t+2L\varepsilon)
            \leq N(\varepsilon,K,||\cdot||_2)2 e^{-\frac{nt^2}{2}} 
        \end{align*}
        where $L$ is the constant of Lipschitzness of Proposition~\ref{prop:lipschitz-depth}.
    \end{cor:bounded-set}    

\begin{proof}
    Since the sigmoid is bounded by $1$ and:
    \begin{equation}\label{eq:depth2GK}
        |\SphereD(z|\samples)- \SphereD(z|\bX)| 
                        \leq \sup_{f\in \cF_z} \lvert \bbE_\bX[f(X)]- \bbE_{\samples}[f(X)] \rvert,
    \end{equation}
    Proposition~\ref{prop:bounded-rademacher} gives us, for $z\in\bbR^d$, the point-wise depth concentration:
    \begin{equation}\label{eq:pointwise-concentration}
        \bbP(|\SphereD(z|\samples)- \SphereD(z|\bX)|t>\Radem_{\bX,n}(\mathcal{F})+t)\leq 2 e^{-\frac{nt^2}{2}}
    \end{equation}
    For $\varepsilon>0$, pick a $\varepsilon$-cover $\cC$ of $K$ with $N(\varepsilon,K,||\cdot||_2)$ elements (a finite number since $K$ is bounded).
    \begin{equation*}
        \forall z \in K, \exists z'\in\cC, \bigg\vert |\SphereD(z|\samples)- \SphereD(z|\bX)|-|\SphereD(z'|\samples)- \SphereD(z'|\bX)|\bigg\vert \leq 2L\epsilon
    \end{equation*}
    where $L$ is the constant of Lipschitzness of Proposition~\ref{prop:lipschitz-depth}.
    Applying \eqnref{eq:pointwise-concentration} to all the elements of the cover with a union bound gives the result.\qed
\end{proof}

Now we have enough results to prove our main one:

\newtheorem*{thm:main}{Theorem \ref{thm:main}}
    \begin{thm:main}
    For any $r, s > 0$, the sphere depth $\SphereD_s$ verifies:
        \begin{equation}
        \tag{\ref{eq:consistency}}
             \underset{n\to+\infty} {\lim} |\SphereD_s(z|\samples)- \SphereD_s(z|\bX)| = 0
        \end{equation}
    and:
        \begin{equation}
                \tag{\ref{eq:concentration}}
            \bbP(|\SphereD_s(z|\samples)-\SphereD_s(z|\bX)|> R_{\bP_{\bX},n}(\cF_z)+t)\leq 2 e^{-\frac{nt^2}{2}}
        \end{equation}
    Moreover, if $\bX$ is absolutely continuous with a bounded support then for all $\varepsilon>0$:
        \begin{equation}
        \tag{\ref{eq:consistency_support}}
             \underset{n\to+\infty}{\lim} \bbP(\underset{z\in supp(\bX)}{\sup}|\SphereD_s(z|\samples)- \SphereD_s(z|\bX)| > \varepsilon) = 0
        \end{equation}
    \end{thm:main}

\begin{proof}
    We prove each of the three equations one by one.
    
    Thanks to \eqnref{eq:depth2GK}, the point-wise consistency (\eqnref{eq:consistency}) can be proven by showing the Glivenko-Cantelli property of $\cF_z$:
    \begin{equation*}
        \underset{n\to+\infty} {\lim}\sup_{f\in \cF_z} \lvert \bbE_\bX[f(X)]- \bbE_{\samples}[f(X)] \rvert =0
    \end{equation*}
    This is achieved directly by combining the results of Theorem~\ref{thm:GC} and Proposition~\ref{prop:bracketnum}, hence the consistence.
    The result of \eqnref{eq:concentration} has been showed already just above as \eqnref{eq:pointwise-concentration} in the proof of Corollary~\ref{cor:bounded-set}.

    Finally, to prove \eqnref{eq:consistency_support} for $\bX$ absolutely continuous, we first apply Corollary~\ref{cor:bounded-set} taking the support for $K$ and picking, let's say $\epsilon=1/n$:
    \begin{align*}
            \bbP(\sup_{z\in supp(\bX)}|\SphereD_s(z|X_n)-\SphereD_s(z|\bX)|>R_{P_X,n}(\mathcal{F}_z)+t+\frac{2L}{n})
            \leq N(1/n,supp(\bX),||\cdot||_2)2 e^{-\frac{nt^2}{2}} \\
            \leq \Theta(n^d)2 e^{-\frac{nt^2}{2}}
        \end{align*}
    Now for $\varepsilon'>0$, since by Proposition~\ref{prop:rademacher-fz} $R_{P,n}(\mathcal{F}_z) = O\left(\sqrt{\frac{d}{n}}\right)$, it tends to zero as $n$ goes to infinity, and $\frac{2L}{n}$ as well, therefore there exists $M>0$ such that for all $n>M$, $\frac{2L}{n}, R_{P,n}(\mathcal{F}_z) <\varepsilon'/3$. Picking $t=\varepsilon'/3$, we obtain:
    \begin{align*}
            \bbP(\sup_{z\in supp(\bX)}|\SphereD_s(z|X_n)-\SphereD_s(z|\bX)|>\epsilon')
            &\leq \bbP(\sup_{z\in supp(\bX)}|\SphereD_s(z|X_n)-\SphereD_s(z|\bX)|>R_{P_X,n}(\mathcal{F}_z)+t+\frac{2L}{n})\\
            &\leq \Theta(n^d)2 e^{-\frac{n(\varepsilon'/3)^2}{2}}
        \end{align*}
    and this quantity goes to zero as $n$ tends to infinity, which concludes the proof.\qed
    
\end{proof}

\subsection{Proofs for section~\ref{subsec:homogeneity-tests}}

\newtheorem*{prop:A3}{Proposition \ref{prop:A3}}
    \begin{prop:A3}
        For $\bX$ absolutely continuous with bounded support:
        \begin{equation*}
            \mathbb{E}(\sup_{supp(\bX)}|\SphereD_s(z|\samples)-\SphereD_s(z|\bX))|)= O\left( \frac{\log(\sqrt{n})}{\sqrt{n}} \right)
        \end{equation*}
    \end{prop:A3}

    \begin{proof}
        For some $\varepsilon>0$ that we will express later, pick an $\varepsilon$-cover $\cC$ of the support of $X$: we denote the elements of the cover $\cC=\{z_1,...,z_N\}$ where $N$ is short for $N(\varepsilon,supp(X),||\cdot||)$. Then we can approximate by $L\varepsilon$ (where $L$ is the Lipschitz constant of Proposition~\ref{prop:lipschitz-depth}) the depth of any point of the support by the depth of the closest $z_i$. If we define $U_i=(|\SphereD(z_i|\samples)-\SphereD(z_i|\bX)|-R_{P_X,n}(\mathcal{F}_z))_+$ (where $(X)_+=max(X,0)$), by applying \eqnref{eq:concentration} of Theorem~\ref{thm:main} to each $z_i$, the $(U_i)$ are subgaussian of rate $1/n$.
        Therefore:
        \begin{equation*}
            \mathbb{E} \max_{i=1,...,n} U_i = \sqrt{2\log(N)/n}
        \end{equation*}
        We deduce, by taking $\varepsilon = \frac{1}{\sqrt{n}}$,        
        \begin{align*}
            \mathbb{E} \sup_{supp(X)}|\SphereD(z|X_n)-\SphereD(z|X)| &\leq  2L\varepsilon + \mathbb{E} \max_{i=1,...,n} U_i +R_{P_X,n}(\mathcal{F}_z) \\
            &= O\left( \frac{\log(\sqrt{n})}{\sqrt{n}} \right)                       
        \end{align*}
        because then $\log N \leq \log ((3\sqrt{n})^d) = \frac{d}{2}\log n + d \log 3$ and $R_{P_X,n}(\mathcal{F}_z)= O\left(\sqrt{\frac{d}{n}}\right)$ by Proposition~\ref{prop:rademacher-fz}. \qed
    \end{proof}

\newtheorem*{th:homogeneity-test}{Theorem\ref{th:homogeneity-test}}
    \begin{th:homogeneity-test}
        Let there be two sets of $n$ and $m$ independent samples respectively coming from a distribution $\bF$ absolutely continuous with bounded support and verifying (A1), each sets of samples with empirical distribution $\bF_n$ and $\bF'_m$ respectively. Then the quality index using $\SphereD_s$ verifies:
        \begin{equation*}
            \left(\frac{1}{12}(\frac{1}{n}+\frac{1}{m})\right)^{-1/2}(Q(\bF_n,\bF'_m)-Q(\bF,\bF))\to \mathcal{N}(0,1)
        \end{equation*}
         in distribution as the number of samples $n,m$ goes to infinity.
    \end{th:homogeneity-test}

    In the domain of such tests, such as in the works of \citet{liu1993quality} and \citet{zuo2006limiting}, when considering using a depth $D$ in such test procedure, it is also implicitly assumed that $D(X|\bF)$ for $X \sim \bF$ is absolutely continuous, otherwise if there is some value $c$ such that the event $D(X|\bF)=c$ is not of measure null but has a certain positive probability, then we cannot guarantee $Q(F,F)=1/2$ ( a trivial counterample would be an absurd constant depth or a Dirac distribution that leads to $Q(F,F)=1$), which we require in practice to test. For the Sphere depth, with an absolutely continuous distribution $F$, with the continuity of the sigmoid, the assumption seems reasonable in practice. Here in the proof we will also implicitly assume $D(X|\bF_n)$ and $D(X|\bF'_m)$ have absolutely continuous values for $X\sim F$, otherwise we need the assumption (A4) of the work of \citet{zuo2006limiting}:
    (A4) $\mathbb{E}(\sum_i p_{iX}(F_m)p_{iY}(F_m)=o(m^{-1/2}))$ if there exist $c_i$ such that $p_iX(F_m)>0$ and $p_iX(F_m)>0$ for $p_{iZ}:=\bbP(D(Z|F_m)=c_i|F_m),i=1,2,...$.

    \begin{proof}
        We follow the proof of \citet{zuo2006limiting}. Following their notations, for a depth $D$ and distributions $\bH,\bF_1,\bF_2$ in $\bbR^d$ ,we note:
        \begin{align*}
            I(x,y,\bH) = \mathds 1_{\{D(x|\bH)\leq D(y|\bH)\}}\\
            I(x,y,\bF_1,\bF_2) = I(x,y,\bF_1)-I(x,y,\bF_2)
        \end{align*}

    They first prove their Lemma 1, that we restate here for completeness:
\newtheorem*{lemma1zuo}{Lemma}    
    \begin{lemma1zuo}
    For $\bF_m$ and $\bG_n$ empirical distributions based on independent samples of sizes $m$ and $n$ from distributions $\bF$ and $\bG$ respectively:
    \begin{itemize}
        \item (i) $\int\int I(x,y,\bF)d(\bG_n(y)-\bG(y))d(\bF_m(x)-\bF(x)) = O(1/\sqrt{mn})$ in probability
        \item (ii) $\int\int I(x,y,\bF_m,\bF)d(\bF_m(x)-\bF(x))d\bG(y) = O(1/\sqrt{m})$ in probablity under (A1)-(A2'), and
        \item (iii) $\int\int I(x,y,\bF_m,\bF)d\bF_m(x)d(\bG_n-\bG)(y) = O(m^{-1/4}n^{-1/2})$ in probability under (A1) and (A3')
    \end{itemize}
        where (A1) is the same as ours and :
    \begin{itemize}
        \item (A2') $\sup_{x\in\mathbb{R}^d} |D(z|\samples)-D(z|\bX)|=o(1)$ almost surely as $n\to\infty$
        \item (A3') $\mathbb{E}(\sup_{supp(X)}|D(z|\bX)-D(z|\bX)|)=O\left(\frac{1}{\sqrt{n}} \right)$     
    \end{itemize}
        
    \end{lemma1zuo}
    Since the integrals only involves terms $x$ and $y$ drawn from distribution $\bF,\bG$, we argue that we can restrain the supremum only to the support instead of all $\bbR^d$. More in details, they use the supremum in their proof when claiming:
    \begin{equation*}
        |I(x,y,\bF_m,\bF)|\leq \mathds 1_{\{ | D(x|\bF)-D(y|\bF)| \leq 2 \sup_{x\in \bbR^d}|( D(x|\bF_m)-D(x|\bF)|\}}.
    \end{equation*}
    But in our case, we assume $\bF=\bG$ and their empirical counterparts have their support included in the population support, it is enough to have:
    \begin{equation*}
        \forall x,y\in supp(\bF), |I(x,y,\bF_m,\bF)|\leq \mathds 1_{\{ | D(x|\bF)-D(y|\bF)| \leq 2 \sup_{x\in supp(\bF)}|( D(x|\bF_m)-D(x|\bF)|\}}.
    \end{equation*}
    Then they proceed to prove that the square of the integral in (iii) is bounded by:
    \begin{equation*}
        \frac{4C}{n}\bbE \left( \sup_{x\in \bbR^d} |( D(x|\bF_m)-D(x|\bF)| \right) = O(1/(m^{1/2} n)).
    \end{equation*}
    Instead, because of our (A3), we will bound by 
    \begin{equation*}
        \frac{4C}{n}\bbE \left( \sup_{x\in supp(\bF)} |( D(x|\bF_m)-D(x|\bF)| \right) = O(log(n)/(m^{1/2} n)).
    \end{equation*}
    and replace claim (iii) (with (A3')) by  claim (iii') (with our (A3)):

    (iii') $\int\int I(x,y,\bF_m,\bF)d\bF_m(x)d(\bG_n-\bG)(y) = O(m^{-1/4}\log(n)^{1/2}n^{-1/2})$ in probability under (A1) and (A3)

    Since $\log(n)$ is still dominated by $n^{1/2}$ and the authors are interested in having negligible terms with respect to $m$ in their proof, it does not change the rest of their proof and their results follow analogously. To give a gist of their proof with more details, they rewrite:
    \begin{align*}
        Q(\bF_m,\bG_n)-Q(\bF,\bG_n) &= \int\int I(x,y,\bF_m,\bF) d\bF_m(x) d\bG_n(y) \\ 
        &+\int\int I(x,y,\bF)d(\bG_n(y)-\bG(y))d(\bF_m(x)-\bF(x))\\
        &+\int\int I(x,y,\bF)d\bG(y)d(\bF_m(x)-\bF(x))
    \end{align*}
    They use their Lemma 1 (iii) to show that the first integral is equivalent to $\int\int I(x,y,\bF_m,\bF)d\bF(x)d\bG(y) +o_p(1/\sqrt{m})$. Then since we assume $\bF=\bG$, because we also assumed the absolute continuity of the depth values, this term is worth $1/2-1/2=0$, or alternatively, assuming (A4), \citet{zuo2006limiting} showed it is a negligible $o(m^{-1/2})$. The rest is proven using their Lemma 1 (i) and the central limit theorem, concluding by rewriting:
    \begin{equation*}
        Q(\bF_m,\bG_n)-Q(\bF,\bG) = (Q(\bF_m,\bG_n)-Q(F,\bG_n))+(Q(F,\bG_n)-Q(\bF,\bG)).
    \end{equation*}\qed
    \end{proof}
    
\section{EXPERIMENTAL RESULTS}\label{app:exp}


    Experiments were carried out on a Macbook Pro 2020 (2,3 GHz Intel Core i7 processor with 4 cores, RAM of 16 Go 3733 MHz LPDDR4X).    

    \subsection{Kendall-$\tau$}
    Here we display the results of the same experiments as section~\ref{subsec:simulated-ranks} in the two figures below, but using the Kendall $\tau$ rank correlation coefficient instead of Spearman's.
                \begin{figure}[!h]\label{fig:kendalltau}
                \begin{subfigure}{0.475\textwidth}
        \centering
        \includegraphics[width=\linewidth]{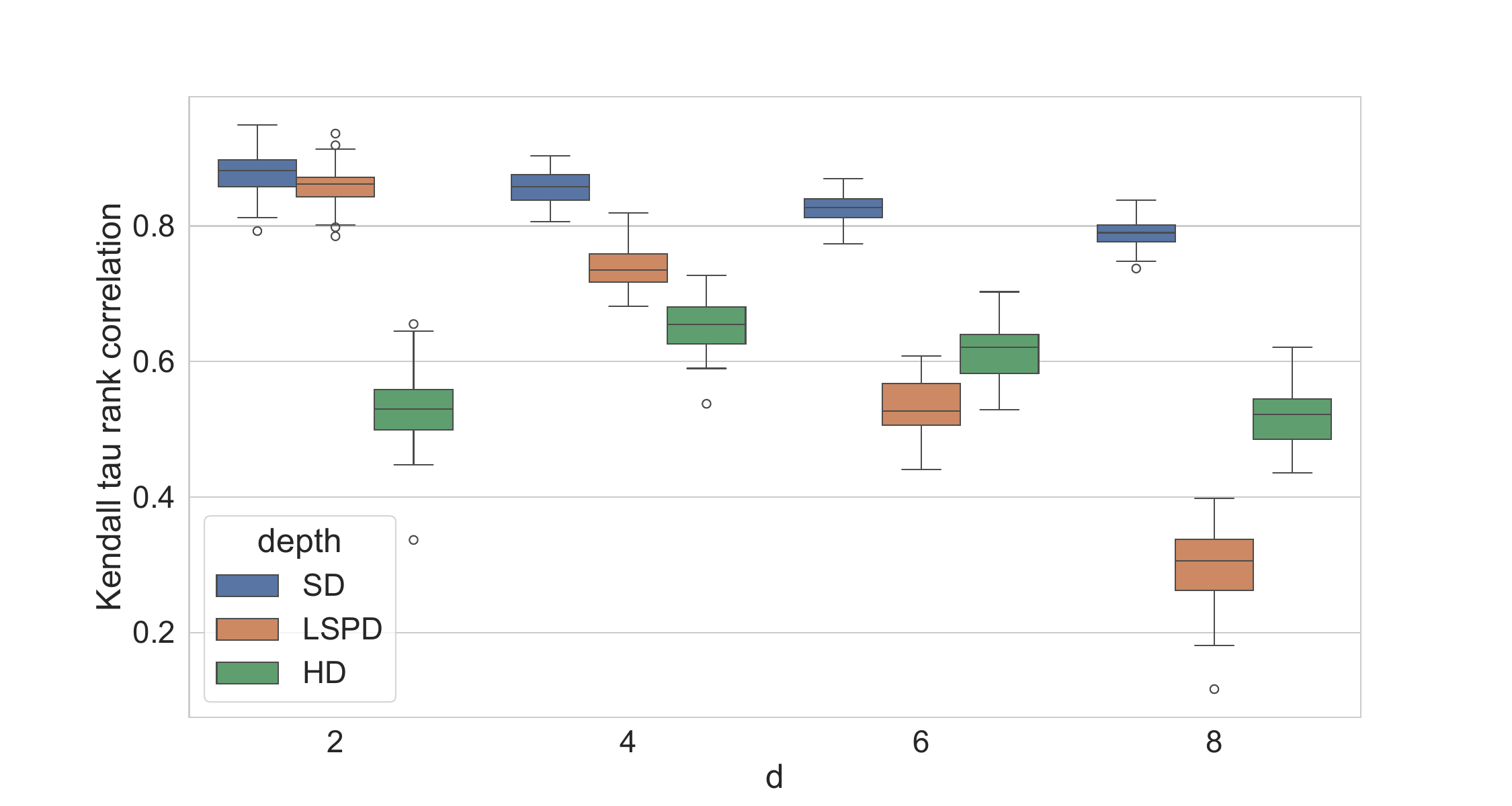}
                         \caption{Bi-Gaussian $SD^1_1$ ($r,s=1$) vs LSPD ($h=1$) Kendall tau correlation ranking w.r.t. true density (50 runs)}
                     \label{fig:dgrowing-kendall}
                \end{subfigure}
        \hfill
                \begin{subfigure}{0.475\textwidth}
        \centering
        \includegraphics[width=\linewidth]{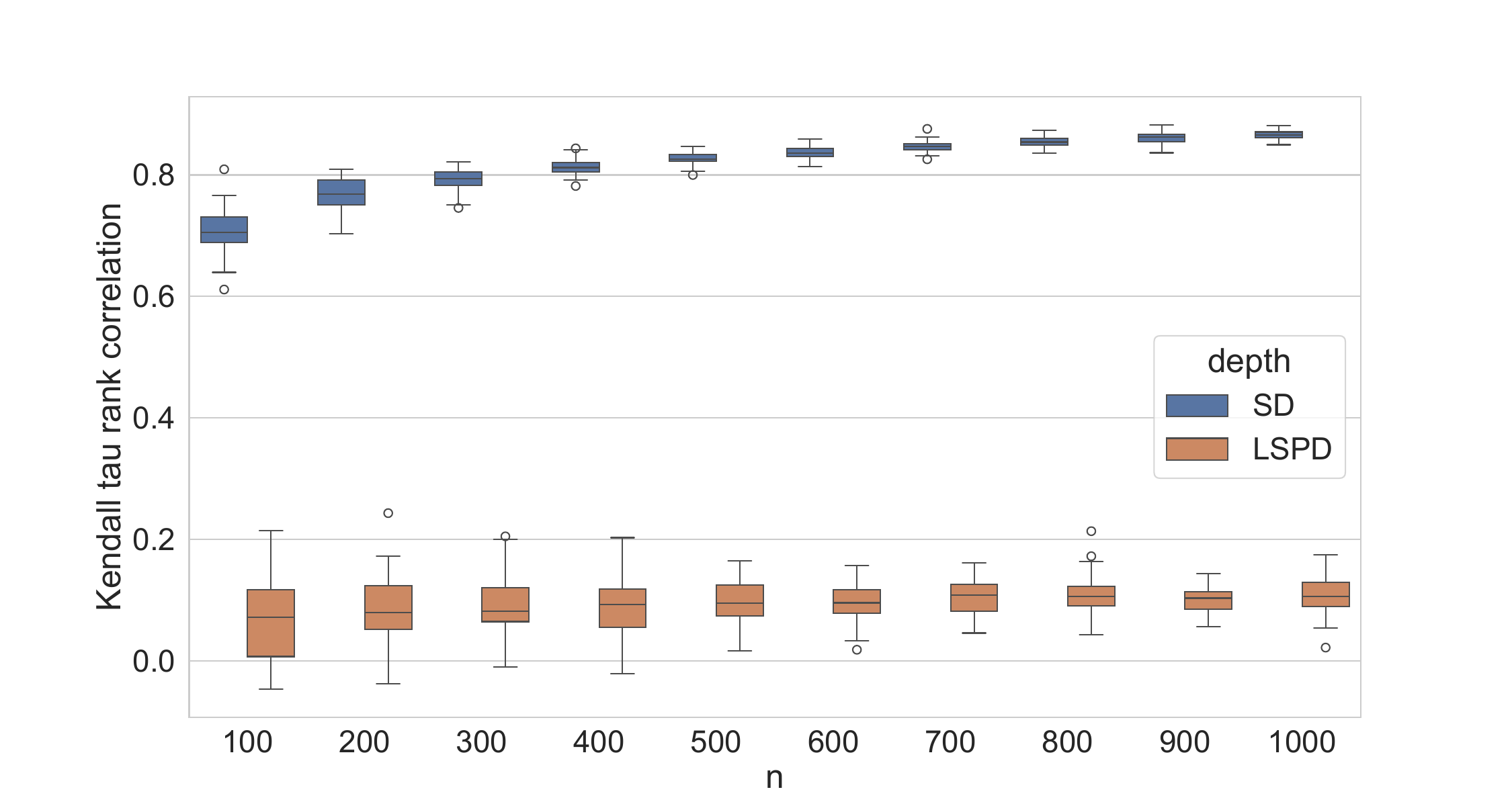}
                         \caption{Convergence of the Kendall tau rank correlation w.r.t. the true density for a bi-Gaussian distribution, d = 10}
                         \label{fig:ngrowing-kendall}
                \end{subfigure}      
                \end{figure}

    \subsection{H-score}
    Here we display the results of the same experiments of anomaly detection of section~\ref{subsec:simulated-ranks} but using the $h$-measure instead of the AUROC. The $h$-measure~\citet{hand2009measuring,hand2014better} is supposed to be better than the AUROC to compare different kinds of algorithms, as it assumes some fixed distribution on the ratio of the costs of false positives and false negatives, and consequently deduces threshold values, while the AUROC by integrating under the curve implicitly integrates on several thresholds corresponding to different cost ratios for each algorithm.
    In Table~\ref{tab:hmeasure}, we show the $h$-scores of the experiments, as well as details on the datasets such as the number of samples, dimension, and anomaly rate.

\begin{table} [h]
\centering
    \caption{$h$-measure on anomaly detection datasets}\label{tab:hmeasure}
    \begin{tabular}{l|rrr|rrrrr}
    \hline
Dataset &     n &   d & Anomaly rate & {OCSVM} & {LOF} & {HD} & {LSPD} & {SD} \\
\hline
wine & 129 &  13 &     0.077519 & 0.37 & 0.74 & 0.00 & 0.10 & \textbf{0.84} \\
glass & 214 &   9 &     0.042056 & \textbf{0.61} & 0.35 & 0.05 & 0.00 & 0.40 \\
vertebral & 240 &   6 &        0.125 & \textbf{0.15} & 0.02 & 0.04 & 0.12 & 0.04 \\
vowels & 1456 &  12 &     0.034341 & 0.46 & \textbf{0.79} & 0.03 & 0.38 & 0.20 \\
pima &  768 &   8 &     0.348958 & 0.07 & \textbf{0.12} & 0.03 & \textbf{0.12} & \textbf{0.12} \\
breastw &  683 &   9 &     0.349927 & 0.75 & 0.20 & 0.47 & 0.51 & \textbf{0.84} \\
lympho &   148 &  18 &     0.040541 & 0.51 & \textbf{0.57} & 0.00 & 0.01 & 0.51 \\
thyroid & 3772 &   6 &     0.024655 & \textbf{0.85} & 0.25 & 0.38 & 0.84 & \textbf{0.85} \\
annthyroid &  7200 &   6 &     0.074167 & \textbf{0.39} & 0.02 & 0.07 & \textbf{0.39} & 0.35 \\
pendigits &  6870 &  16 &     0.022707 & 0.25 & 0.05 & 0.03 & 0.06 & \textbf{0.39} \\
cardio &  1831 &  21 &     0.096122 & \textbf{0.47} & 0.04 & 0.00 & 0.00 & 0.46 \\
\hline
\end{tabular}
\end{table}

\section{A NOTE ON COMPLEXITY OF THE ALGORITHM}
The input data take space $O(nd)$, and the additionally used buffer is of size $O(d)$.
Each step of the gradient descent needs to compute the expectation on $n$ samples, and takes time $O(nd)$.
Therefore the complexity of Algorithm~\ref{alg:grad-ssd} in the worst case is of $O(n d n_{iter})$.

\newpage


\end{document}